\documentclass{article}

\usepackage[preprint]{neurips_2026}

\usepackage{iftex}
\ifPDFTeX
  \usepackage[utf8]{inputenc}
  \usepackage[T1]{fontenc}
\fi

\usepackage{hyperref}
\usepackage{url}
\usepackage{booktabs}
\usepackage{amsfonts}
\usepackage{amsmath,amssymb,amsthm}
\usepackage{mathtools}
\usepackage{nicefrac}
\usepackage{xcolor}
\usepackage{graphicx}
\usepackage{wrapfig}
\usepackage{xspace}
\usepackage{comment}

\usepackage{algorithm}
\usepackage{algorithmic}

\IfFileExists{DL_notations.tex}{
  \input{DL_notations} 
}{
  \usepackage{bm}

\def\vtheta{{\bm{\theta}}}

\renewcommand{\vec}[1]{\boldsymbol{#1}}

\renewcommand{\vtheta}{\vec{\theta}}

}

\newcommand{\LLQR}{LLQR\xspace}
\newcommand{\SAM}{SAM\xspace}
\newcommand{\LLQRSAM}{LLQR+SAM\xspace}
\definecolor{LLQRSAMNGD}{HTML}{1F77B4}
\definecolor{LLQRSAMNewton}{HTML}{B35C00}

\theoremstyle{plain}
\newtheorem{theorem}{Theorem}[section]
\newtheorem{proposition}[theorem]{Proposition}
\newtheorem{corollary}[theorem]{Corollary}

\theoremstyle{remark}

\newtheorem{assumption}[theorem]{Assumption}

\title{Navigating Potholes with Geometry-Aware Sharpness Minimization}

\author{%
  Simon Dufort-Labbé\\
  Mila, Université de Montréal\\
  \And
  Mehrab Hamidi\\
  Mila, Université de Montréal
  \And
  Razvan Pascanu\\
  Mila, Université de Montréal
  \And
  Ioannis Mitliagkas\\
  Mila, Université de Montréal
  \And
  Damien Scieur\\
  Samsung -- SAIL Montreal\\
  Mila, Université de Montréal\\
  \And
  Aristide Baratin\\
  Samsung -- SAIL Montreal\\
  Mila, Université de Montréal\\
}

\begin{document}

\maketitle

\begin{figure*}[b!]
  \centering
  \includegraphics[width=0.98\textwidth]{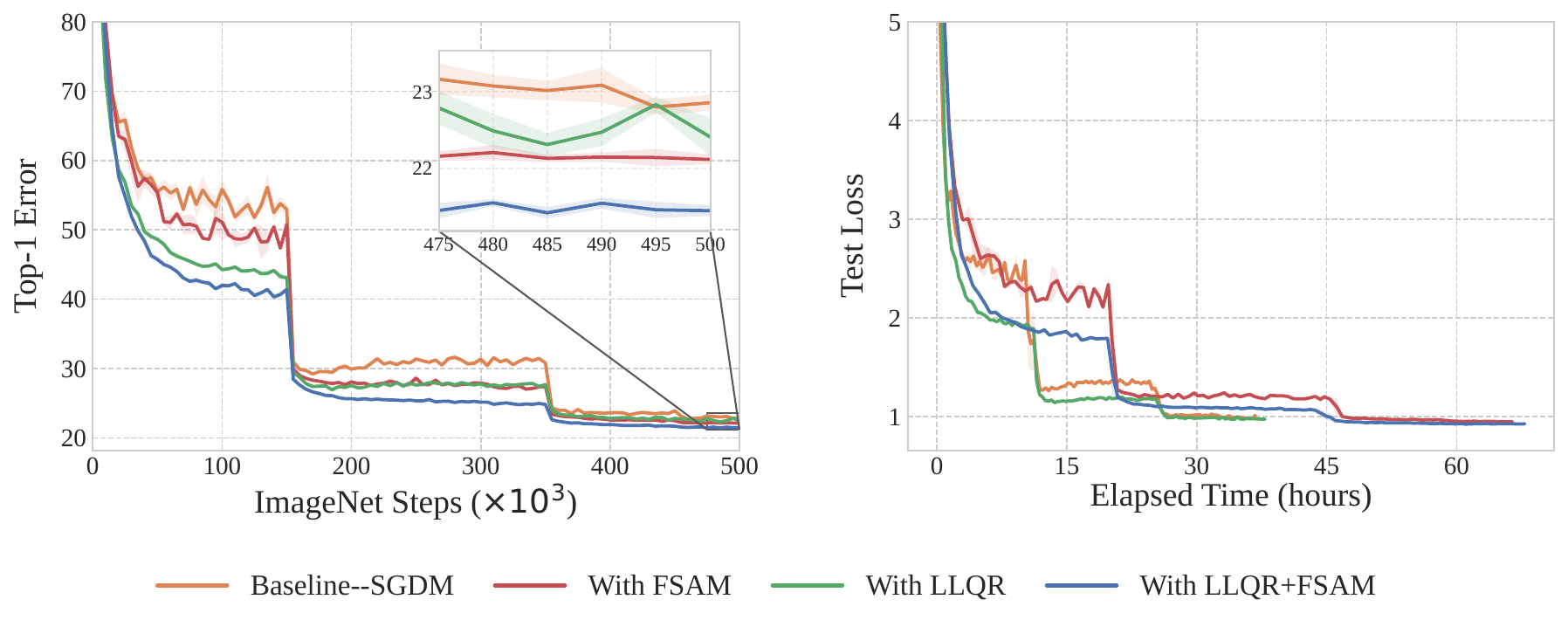}
  \caption{\textbf{Interaction between FSAM and \LLQR\ on ResNet-50/ImageNet.}
  \textbf{Left}: Top-1 error for SGDM and FSAM, a \SAM variant, with and without \LLQR. Although both \SAM-style methods and \LLQR\ can be interpreted as curvature-correcting mechanisms, their combination yields gains over either component alone, suggesting complementary effects. \textbf{Right}: Test loss versus elapsed training time. Despite the usual concern that second-order methods are prohibitively expensive, \LLQR\ remains efficient enough not to become the bottleneck: whenever the slowdown of \SAM\ is acceptable, topping it with \LLQR\ likely preserves affordability.}
  \label{fig:imagenet_fsam_top1_loss_time}
\end{figure*}
\vspace{-1.5em}
\begin{abstract}
Sharpness-aware minimization (SAM) encourages flat minima by perturbing parameters along directions of high loss curvature, but treats all parameter directions uniformly, ignoring the underlying loss geometry. We introduce LLQR+SAM, which combines SAM with a learned preconditioner obtained from the recently proposed LLQR framework--a second-order method that recasts steepest descent as a layerwise linear-quadratic regulator problem. The preconditioner is updated sparsely and maintained as a slow exponential moving average, so it captures a smoothed, low-resolution picture of the loss landscape geometry. The SAM perturbation then operates on top of this learned geometry, probing curvature at a faster timescale. We show that this two-timescale structure is not merely a computational convenience: theoretically, the preconditioner amplifies the SAM escape signal in directions that are flat under the average geometry but locally sharp (“potholes”).  Wide, flat basins, by contrast, remain stable. Empirically, LLQR+SAM gives consistent gains over both SAM and LLQR alone across standard vision and sequence modeling benchmarks, supporting the view that slow learned geometry and fast sharpness correction are genuinely complementary.
\end{abstract}


\section{Introduction}

Modern deep-learning loss landscapes are widely thought to combine smooth,
slowly-varying structure at large scales---wide basins, valleys, and ridges
shaped by the architecture and data---with much sharper features at smaller
scales: narrow valleys, spurious minima, and localized \emph{potholes} that
act as attractors of poor generalization. Two largely independent lines of
optimization research target these two scales. \emph{Geometry-aware} methods
such as natural gradient~\citep{amari1998natural}, Newton, K-FAC~\citep{martens2015kfac}
and Shampoo~\citep{gupta2018shampoo} use estimates of curvature to reshape the
descent direction, addressing the ill-conditioning of the average loss
geometry. \emph{Sharpness-aware} methods, beginning with
SAM~\citep{foret2021sam}, perturb the parameters before evaluating the
gradient, biasing optimization toward flatter minima.

These two mechanisms operate on different aspects of the geometry, and at
very different time scales. A learned preconditioner accumulated over many
gradient steps can only capture the slow average curvature of the
landscape; it is by construction insensitive to sharp localized features.
A SAM perturbation, evaluated on a single minibatch, probes local sharpness
directly---but treats all parameter directions uniformly, ignoring the
underlying optimization geometry. Each captures something the other misses.

In this work, we study the combination of the two. We use
\textbf{LLQR}~\citep{Anonymous2026LLQR}, a recently proposed framework that
learns a structured inverse preconditioner $U$ by minimizing a
linear-quadratic-regulator (LQR) objective derived from the layerwise
network dynamics and a chosen divergence on the loss. The preconditioner
is updated sparsely (a few times per epoch) and maintained as an
exponential moving average; it captures a smoothed, low-resolution picture
of the average loss geometry. On top of this slow geometry we apply a SAM
perturbation evaluated on each minibatch \emph{in the $U$-induced norm},
and the resulting outer update uses the same learned geometry. The combined
algorithm, \textbf{LLQR+SAM}, applies the SAM displacement and the gradient
transport in a single learned metric.

Our central claim is that this two-time-scale combination is
\emph{synergistic}, not merely additive. We make this precise on a model
landscape consisting of a quadratic loss with a smooth average geometry,
supplemented by a sharp localized pothole that creates a spurious minimum
invisible to $U$. On this landscape the LLQR+SAM dynamics is closed-form
and exposes a clean
mechanism: SAM prevents the iterate from settling at any minimum, forcing
it to hover at the scale of the SAM probe itself, while $U$ shapes that
scale according to the average geometry. Around a wide basin the hovering
is appropriate to the basin width and the iterate is stable; around a
pothole the same hovering scale exceeds the pothole's basin of attraction
and the iterate escapes. Crucially, the escape signal is amplified relative
to vanilla SAM by a factor that grows as the surrounding basin becomes
flatter---precisely the regime where preconditioning matters most. The
same $U$ that enables fast convergence in flat directions also enables
pothole escape: $U$ and SAM are not independent ingredients but two
complementary uses of one learned geometry. These predictions are borne
out by the toy experiments in Figures~\ref{fig:toy_sam_escape_comparison} and~\ref{fig:toy_sam_gradient_noise_escape_comparison} and by the empirical study of
Section~\ref{sec:experiments}, which shows consistent gains over either
component alone across CIFAR, TinyImageNet, ImageNet, and IWSLT14.

\paragraph{Contributions.}
(i) We propose LLQR+SAM, combining a slowly-updated, EMA-smoothed
second-order preconditioner with a SAM perturbation evaluated in the
induced geometry. (ii) On a quadratic landscape with a sharp pothole
perturbation we derive the exact LLQR+SAM dynamics and quantify the
escape-amplification of the preconditioner relative to vanilla SAM, with
the gain growing as the surrounding basin becomes flatter. (iii) We
validate empirically that the combination yields consistent gains over
either component alone across standard vision and sequence benchmarks.


\section{Related Work}
\label{sec:related_work}

\paragraph{Sharpness, flatness, and SAM.}
The relation between generalization and loss-landscape sharpness has long motivated methods that bias training toward flatter minima~\citep{hochreiter1997flat,keskar2017largebatch,neyshabur2015path}, although this connection is neither universal nor invariant: sharpness can change under reparameterization without changing the represented function~\citep{dinh2017sharp}. \SAM~\citep{foret2021sam} is the canonical sharpness-aware optimizer, replacing the training objective by a local min--max problem that probes the worst-case loss near the current parameters and descends using the gradient at the perturbed point. Many variants modify either the perturbation neighborhood or the adversarial-point estimate: ASAM~\citep{kwon2021asam} uses a scale-adaptive radius to reduce sensitivity to parameter rescaling, while Friendly SAM~\citep{li2024friendly} reduces minibatch-noise sensitivity by averaging perturbation information across batches. A complementary view comes from gradient-norm penalization: \citet{zhao2022penalizing} argue that penalizing the gradient norm encourages smaller local Lipschitz constants and show that \SAM\ can be recovered as a special case of a loss augmented by a gradient-norm regularizer. 
\paragraph{Geometry-aware sharpness neighborhoods.}
Vanilla \SAM\ defines sharpness through Euclidean neighborhoods in parameter space, which may poorly reflect the effective geometry of neural networks~\citep{li2018visualizing}. Fisher SAM addresses this by replacing the Euclidean ball with an ellipsoid induced by the Fisher information matrix, so the adversarial probe is measured in a geometry tied to the model distribution~\citep{kim2022fishersam}. Riemannian SAM further generalizes this idea by defining the sharpness neighborhood through an arbitrary Riemannian metric~\citep{yun2023riemanniansam}. In this view, Fisher SAM is a special case, and sharpness-aware perturbations are interpreted as steepest-ascent directions under a non-Euclidean local norm. 
\paragraph{Geometry-aware optimization and robustness.}
A separate line of work uses geometry to transport gradients rather than to define adversarial neighborhoods. Natural-gradient descent defines steepest descent with respect to the Fisher information metric~\citep{amari1998natural}. K-FAC makes this principle scalable through layerwise Kronecker-factored Fisher approximations~\citep{martens2015kfac}, while Shampoo maintains tensor-structured preconditioners along parameter dimensions~\citep{gupta2018shampoo}. These methods use curvature 
estimates to precondition the optimization trajectory, but they do not introduce a \SAM-style local adversarial probe. Another related direction controls sensitivity to input perturbations rather than parameter perturbations: spectral normalization and Lipschitz regularization reduce the local sensitivity of the network function to data-space changes~\citep{yoshida2017spectral,virmaux2018lipschitz}. Thus, while geometry-aware optimizers act on the parameter-space training dynamics and Lipschitz methods target input-space robustness, sharpness-aware methods study loss variation under parameter perturbations.


\section{Method}
\label{sec:method}

\paragraph{Metric and notation.}
Let $L:\mathbb{R}^d\!\to\!\mathbb{R}$ denote the training loss, with
$g(\theta):=\nabla L(\theta)$ and $H(\theta):=\nabla^2 L(\theta)$. We write
$P(\theta)\succ0$ for a local optimization metric and
$U(\theta):=P(\theta)^{-1}$ for its inverse. The associated dual norm on gradients is 
\[
    \|g\|_U := (g^\top U g)^{1/2}.
\]
\textbf{LLQR inverse metric.}
The LLQR framework~\citep{Anonymous2026LLQR} (See Appendix~\ref{app:llqr-summary}) learns a structured inverse
preconditioner $U$ from layerwise network dynamics and a divergence on the
loss, such as KL/Fisher or Newton/Bregman. In practice, $U$ is represented by
diagonal or K-FAC
blocks,
trading expressivity for cost. It is updated
only every $n_{\text{LLQR}}\!\sim\!500$ optimizer steps and smoothed by an
exponential moving average, so applying $U$ during training is a structured
matrix-vector product with modest overhead. For the analysis in
Section~\ref{sec:analysis}, the relevant properties are that $U_t$ is positive
definite, uniformly spectrally bounded, and frozen during each
perturb-and-recompute step.

\textbf{Euclidean SAM.}
Standard SAM~\citep{foret2021sam} evaluates the update gradient at the
worst-case first-order perturbation in a Euclidean ball:
\[
\theta^+
=
\theta
+
\rho\,\frac{g(\theta)}{\|g(\theta)\|_2},
\qquad
\theta
\leftarrow
\theta
-
\eta\,g(\theta^+).
\]

\paragraph{LLQR+SAM.}
LLQR+SAM replaces the Euclidean perturbation geometry by the learned metric
$P=U^{-1}$ and uses the same inverse metric to transport the outer gradient:
\begin{equation}
\theta^+
=
\theta
+
\rho\,\frac{U\,g(\theta)}{\|g(\theta)\|_U},
\qquad
\theta
\leftarrow
\theta
-
\eta\,U\,g(\theta^+).
\label{eq:llqrsam}
\end{equation}
The perturbation is therefore the normalized steepest-ascent direction under
the metric $P$, while the descent step applies the same geometry to the
gradient evaluated at the perturbed point.

\textbf{Interpretation.}
Equation~\eqref{eq:llqrsam} can be viewed locally as a Riemannian \SAM\ step in
the LLQR metric, with $U_t$ fixed during the perturb-and-recompute step. The
corresponding fixed-metric transfer from Riemannian
\SAM~\citep{yun2023riemanniansam} gives stationarity in the instantaneous
metric; Appendix~\ref{app:proof_riemannian_transfer} gives the formal
statement. This also separates LLQR+SAM from Fisher
SAM~\citep{kim2022fishersam}: Fisher SAM uses a Fisher-scaled perturbation to
probe sharpness, but typically applies the resulting gradient with the base
optimizer. LLQR+SAM instead uses the learned inverse metric $U_t$ both to
define the perturbation neighborhood and to transport the outer gradient. This
coupling is central to the pothole-escape mechanism analyzed in
Section~\ref{sec:analysis}: the amplification depends on the metric acting on
both the probe and the descent step.

\textbf{Algorithm}
\begin{algorithm}[t]
\caption{\LLQRSAM}
\label{alg:llqrsam}
\begin{algorithmic}[1]
\STATE \textbf{Input:} initial parameters $\vtheta_0$, learning rate $\eta$, \SAM\ radius $\rho$, \LLQR\ metric state, damping and update schedule.
\FOR{$t = 0,1,2,\dots$}
  \STATE Draw minibatch $b_t$ and compute $g_t=\nabla_{\vtheta}L_{b_t}(\vtheta_t)$.
  \STATE Update or query the slow \LLQR\ state to obtain $U_t\approx P_t^{-1}$.
  \STATE Set $\epsilon_t=\rho\,U_tg_t/(g_t^\top U_tg_t)^{1/2}$.
  \STATE Form the probe parameters $\vtheta_t^+=\vtheta_t+\epsilon_t$.
  \STATE Compute the sharpness-aware gradient $\tilde g_t=\nabla_{\vtheta}L_{b_t}(\vtheta_t^+)$.
  \STATE Apply the geometry-aware update $\vtheta_{t+1}=\vtheta_t-\eta\,U_t\tilde g_t$.
\ENDFOR
\STATE \textbf{Output:} trained parameters $\vtheta_T$.
\end{algorithmic}
\end{algorithm}
~\ref{alg:llqrsam} states the core update without momentum, weight
decay, or optimizer-specific bookkeeping.  In implementation, these terms are
composed around the same two operations: use the current \LLQR\ state to define
the perturbation geometry, then apply the preconditioned sharpness-aware
gradient through the geometry-aware base update. 

The wall-clock overhead of
LLQR+SAM relative to the corresponding non-SAM, non-preconditioned run is
approximately one extra forward/backward pass per step (the SAM probe), 
plus the cost of refreshing and applying the LLQR preconditioner. As shown in Fig.~\ref{fig:imagenet_fsam_top1_loss_time} (right), this
additional geometry does not become the bottleneck in practice: when the cost of
\SAM is acceptable, adding \LLQR preserves the same affordability regime
while yielding stronger accuracy. For reproducibility, the implementation is available at \href{https://github.com/SimonDufLab/LLQR}{github.com/SimonDufLab/LLQR}.

\section{Analysis: pothole navigation in a two-scale landscape}
\label{sec:analysis}

We analyze \LLQR+\SAM\ on a minimal two-scale quadratic model: an average
geometry tracked by \LLQR, plus a sharp localized perturbation not captured by
the slowly updated preconditioner.

\paragraph{Two-scale model.}
Place a local minimum at the origin and consider
\begin{equation}
L(\theta)
=
\tfrac12\,\theta^\top H\theta,
\qquad
H=\bar H+H_\epsilon,
\qquad
H\succ0,
\label{eq:quad-loss}
\end{equation}
where $\bar H\succ0$ encodes the smooth, slowly-varying part of the geometry
and $H_\epsilon\succeq0$ is a localized sharp component. We make the following
modeling assumption, motivated by the slow, exponential moving average smoothed nature of the \LLQR\
preconditioner.
\begin{assumption}
\label{assum:U}
The \LLQR\ inverse metric captures the average geometry: $U=\bar H^{-1}$.
\end{assumption}
For simplicity, assume that $\bar H$ and $H_\epsilon$ share the same
eigenbasis. Define
\[
    \lambda_i=\bar\lambda_i+\lambda_{\epsilon,i},
    \qquad
    \mu_i:=\lambda_i\bar\lambda_i^{-1}
    =
    1+\lambda_{\epsilon,i}\bar\lambda_i^{-1}.
\]
Thus $\mu_i\approx1$ in average directions and $\mu_i\gg1$ in pothole
directions. The non-commuting case is handled by whitening with
$\bar H^{-1/2}$; see Appendix~\ref{app:whitened_pothole_model}.

\textbf{Dynamics.}
For~\eqref{eq:quad-loss}, $g(\theta)=H\theta$, the LLQR+SAM recursion is exact for the quadratic loss (no Taylor approximation; see Appendix~\ref{sec:coordinate_recursion}):
\begin{equation}
e_{t+1}
=
(I-\eta UH)e_t
-
\eta\rho\,\|He_t\|_U^{-1}UHUH e_t ,
\label{eq:dyn}
\end{equation}
where $e_t=\theta_t$ is the displacement from the minimum. Restricted to a single eigendirection $v_i$ (i.e., $e_t \propto v_i$),
\begin{equation}
z_{i,t+1}
=
(1-\eta\mu_i)z_{i,t}
-
\eta\rho\,\mu_i\bar\lambda_i^{-1/2}
\operatorname{sign}(z_{i,t}),
\qquad
z_{i,t}:=\langle v_i,e_t\rangle .
\label{eq:scalar-dyn}
\end{equation}

\paragraph{Hovering and escape.}
The scalar map~\eqref{eq:scalar-dyn} has no nonzero fixed point: the
\SAM\ perturbation makes the iterate oscillate around the minimum instead of
settling exactly at it. Its hovering envelope is
\begin{equation}
|z_{i,t}|
\lesssim
\rho\,\bar\lambda_i^{-1/2},
\label{eq:hovering}
\end{equation}
with the exact two-cycle amplitude given in
Appendix~\ref{sec:hovering_env}. This scale comes from a cancellation:
the \SAM\ kick is of size
$\eta\rho\mu_i\bar\lambda_i^{-1/2}$, while the contraction scale is
$\eta\mu_i$, hence
\[
\frac{\eta\rho\mu_i\bar\lambda_i^{-1/2}}{\eta\mu_i}
=
\rho\bar\lambda_i^{-1/2}.
\]
Thus the hovering scale is controlled by the average curvature
$\bar\lambda_i$, not by the localized sharpness $\lambda_{\epsilon,i}$.
Consequently, if a pothole has basin radius $r_\epsilon$, the iterate is no longer trapped in the pothole whenever
\begin{equation}
\rho\,\bar\lambda_\epsilon^{-1/2}
>
r_\epsilon .
\label{eq:escape}
\end{equation}

\paragraph{Comparison with vanilla \SAM.}
For vanilla \SAM, the same calculation gives a Euclidean hovering envelope of
order $\rho$ in every direction; see
Appendix~\ref{app:vanilla_sam_comparison}. Therefore \LLQR+\SAM\ amplifies the
pothole-direction escape scale by $\bar\lambda_\epsilon^{-1/2}$.

\begin{corollary}[Pothole-escape amplification; see Cor.~\ref{cor:amplification_app}]
\label{cor:amplification}
Under Assumption~\ref{assum:U} and the commuting hypothesis, the
\LLQR+\SAM\ hovering envelope around a pothole minimum is
\[
\rho\,\bar\lambda_\epsilon^{-1/2}
=
\bar\lambda_\epsilon^{-1/2}
\times
\bigl(\text{vanilla-\SAM\ hovering envelope}\bigr).
\]
The amplification depends only on the average curvature
$\bar\lambda_\epsilon$ in the pothole direction.
\end{corollary}

Thus retuning vanilla \SAM\ cannot reproduce the effect selectively:
increasing $\rho$ enlarges the probe in all directions, whereas \LLQR+\SAM\
uses the same $U\approx\bar H^{-1}$ for direction-adaptive probing and
direction-adaptive contraction.

\paragraph{Toy illustrations.}
Figures~\ref{fig:toy_sam_escape_comparison} and
\ref{fig:toy_sam_gradient_noise_escape_comparison} instantiate the mechanism on
a two-dimensional sharp-well landscape. The non-\SAM\ variants remain trapped,
whereas \SAM\ variants escape. With injected gradient noise, \LLQR+\SAM\
reaches the flat basin with shorter path length than vanilla \SAM, consistent
with Corollary~\ref{cor:amplification}. A stochastic selection model is given
in Appendix~\ref{app:selection}.

\begin{figure*}[t]
  \centering
  \includegraphics[width=0.96\textwidth]{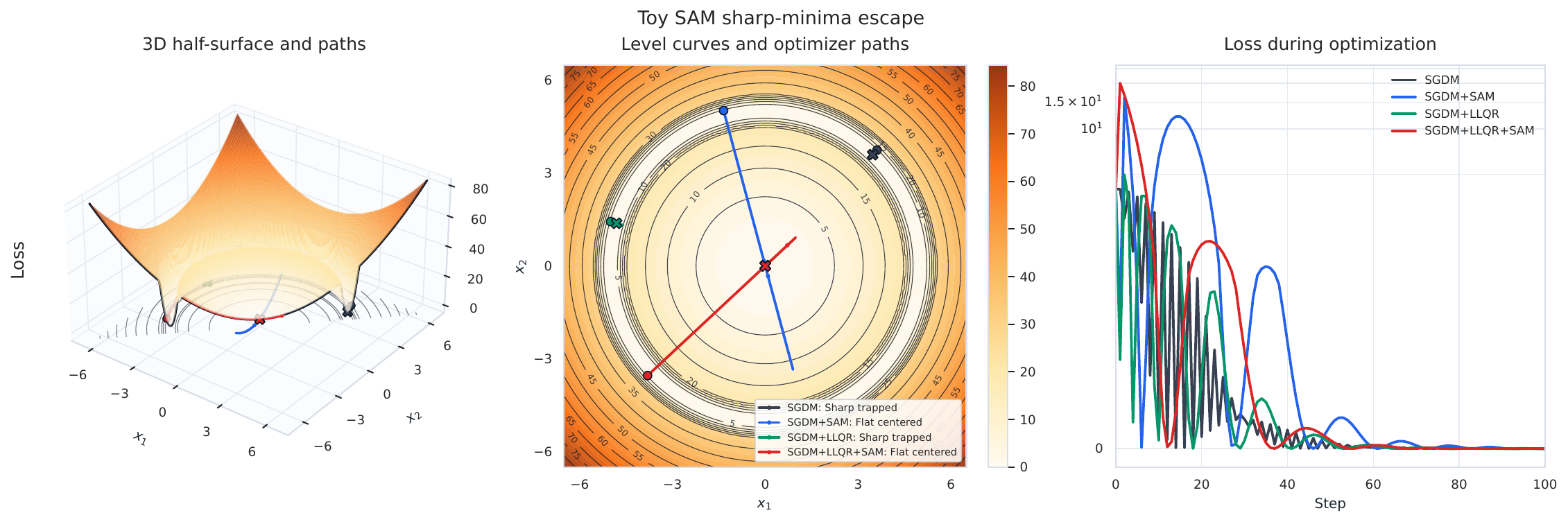}
  \caption{\textbf{Toy sharp-well escape mechanism.} The surface has a flat basin at the origin and a sharp annular basin near radius \(5\). All four optimizers use the same learning rate, and the SAM variants use the same radius, with starts chosen in the sharp basin, but not at minima. The non-SAM variants remain trapped, while the SAM variants leave the sharp well; the \LLQR+\SAM\ trajectory reaches the flat region with faster loss decay in this setting.}
  \label{fig:toy_sam_escape_comparison}
\end{figure*}

\paragraph{Stochastic counterpart: selection between wells.}

The deterministic analysis above shows that pothole minima are not
attracting fixed points of LLQR+SAM. Under stochastic gradients this
refines into a \emph{selection} statement: sharp wells are short-lived
(entered briefly through gradient noise, exited through the SAM kick),
while flat wells support long visits. The same hovering-vs-basin
comparison~\eqref{eq:escape} controls the expected exit time, with
LLQR+SAM benefiting from the amplification of
Corollary~\ref{cor:amplification}. We give the regenerative model and the
exit-time bound in Appendix~\ref{app:selection}; the gradient-noise toy of
Figure~\ref{fig:toy_sam_gradient_noise_escape_comparison} is consistent with the prediction that
LLQR+SAM exits the sharp well faster than vanilla SAM and reaches the
flat basin with a markedly shorter path length.

\begin{figure*}[t]
  \centering
  \includegraphics[width=0.96\textwidth]{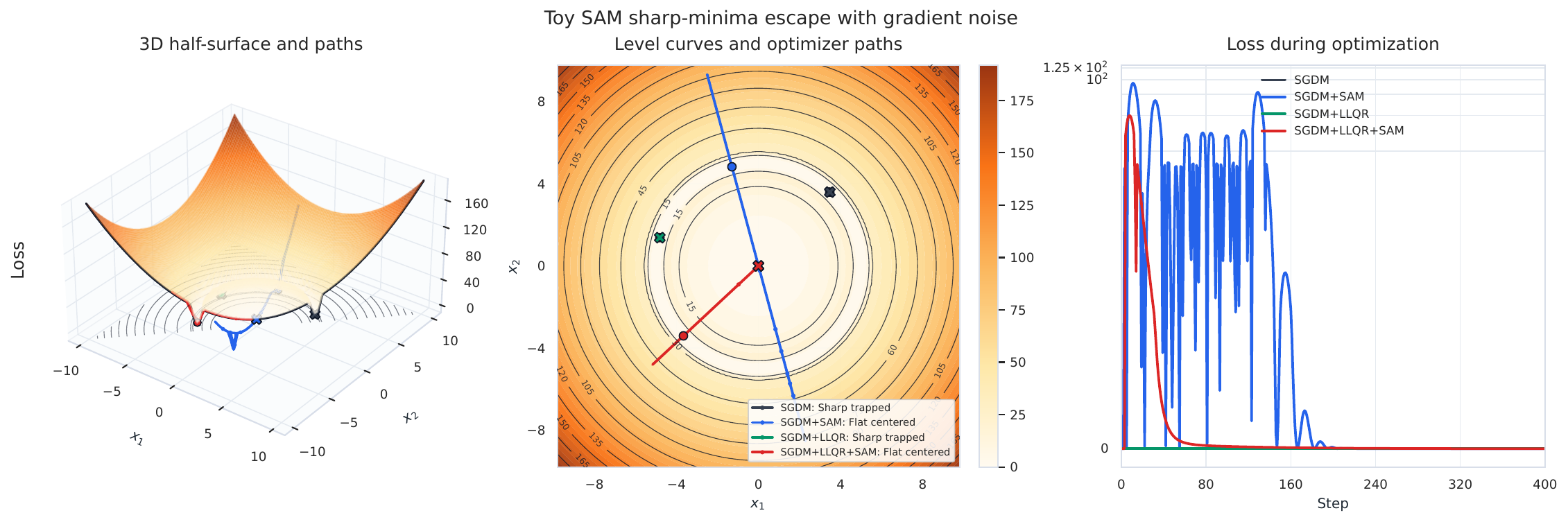}
  \caption{\textbf{Gradient-noise escape from the sharp minimum.} All variants start at the bottom of the sharp well and receive a shared deterministic Gaussian perturbation schedule in the update gradient. The non-SAM variants remain near the sharp well, while SAM variants are ejected. At variance \(10^{-9}\), both SAM variants reach the flat basin, but \LLQR+\SAM\ has substantially shorter path length than Euclidean SAM, consistent with the stochastic selection picture in Appendix~\ref{app:selection}.}
  \label{fig:toy_sam_gradient_noise_escape_comparison}
\end{figure*}


\section{Experiments}
\label{sec:experiments}
In this section, we evaluate LLQR+SAM on standard supervised vision benchmarks (CIFAR-10/100, TinyImageNet, ImageNet), a sequence-modeling task (IWSLT14 De-En), 
and a scalability study on ViT models up to ViT-Large/16. Throughout, LLQR+SAM reuses the hyperparameters of its LLQR and SAM components without method-specific retuning, isolating the effect of the geometry-sharpness coupling. Full training protocols are deferred to Appendix~\ref{app:experiments-details}. Additional experiments about noise injection and method scalability are presentend in Appendix~\ref{app:add_results}. 

\subsection{Standard Supervised Training Benchmarks}

\textbf{CIFAR Datasets.}
We first evaluate the method on CIFAR-10 and CIFAR-100~\citep{krizhevsky2009learning} with
VGG16-BN~\citep{simonyan2015very}, ResNet-18~\citep{he2016deep},
WRN-28-10~\citep{zagoruyko2016wide}, and
PyramidNet-110~\citep{han2017deep}; the full protocol is deferred to
Appendix~\ref{app:standard-supervised-training-details}.
\begin{table}[t]
  \centering
  \caption{CIFAR-10 standard supervised benchmark. Entries are best top-1 test
    accuracy (mean $\pm$ std, 5 seeds). \textcolor{LLQRSAMNGD}{Blue} denotes NGD-induced geometry, \textcolor{LLQRSAMNewton}{Orange} denotes Newton-induced geometry.}
  \label{tab:cifar10-standard-benchmarks}
  \small
  \resizebox{\textwidth}{!}{%
  \begin{tabular}{@{}lccccccc@{}}
    \toprule
    & \multicolumn{3}{c}{Base geometry}
    & \multicolumn{2}{c}{Sharpness baselines}
    & \multicolumn{2}{c}{Combined geometry and sharpness} \\
    \cmidrule(lr){2-4}\cmidrule(lr){5-6}\cmidrule(l){7-8}
    Architecture
      & \textsc{SGDM}
      & \textcolor{LLQRSAMNGD}{\LLQR}
      & \textcolor{LLQRSAMNewton}{\LLQR}
      & \textsc{SAM}
      & \textsc{FisherSAM}
      & \textcolor{LLQRSAMNGD}{\LLQRSAM}
      & \textcolor{LLQRSAMNewton}{\LLQRSAM}\\
    \midrule
    VGG16-BN
      & $95.10 \pm 0.14$
      & $95.27 \pm 0.19$
      & $95.41 \pm 0.11$
      & $95.54 \pm 0.14$
      & $95.45 \pm 0.04$
      & \textcolor{LLQRSAMNGD}{$\textbf{95.92} \pm 0.07$}
      & \textcolor{LLQRSAMNewton}{$95.73 \pm 0.08$} \\
    ResNet-18
      & $96.14 \pm 0.19$
      & $96.37 \pm 0.11$
      & $96.37 \pm 0.11$
      & $96.65 \pm 0.16$
      & $96.74 \pm 0.14$
      & \textcolor{LLQRSAMNGD}{$\textbf{96.85} \pm 0.09$}
      & \textcolor{LLQRSAMNewton}{$96.78 \pm 0.03$} \\
    WRN-28-10
      & $96.94 \pm 0.09$
      & $97.04 \pm 0.08$
      & $97.10 \pm 0.03$
      & $97.43 \pm 0.09$
      & $97.41 \pm 0.04$
      & \textcolor{LLQRSAMNGD}{$\textbf{97.64} \pm 0.08$}
      & \textcolor{LLQRSAMNewton}{$97.52 \pm 0.07$} \\
    PyramidNet-110
      & $97.18 \pm 0.17$
      & $97.21 \pm 0.08$
      & $97.26 \pm 0.10$
      & $97.63 \pm 0.09$
      & $97.57 \pm 0.12$
      & \textcolor{LLQRSAMNGD}{$\textbf{97.86} \pm 0.04$}
      & \textcolor{LLQRSAMNewton}{$97.66 \pm 0.08$} \\
    \bottomrule
  \end{tabular}%
  }
  \vspace{-1.0em}
\end{table}
\begin{table}[t]
  \centering
  \caption{CIFAR-100 standard supervised benchmark. Entries are best top-1 test
    accuracy (mean $\pm$ std, 5 seeds). \textcolor{LLQRSAMNGD}{Blue} denotes NGD-induced geometry, \textcolor{LLQRSAMNewton}{Orange} denotes Newton-induced geometry.}
  \label{tab:cifar100-standard-benchmarks}
  \small
  \resizebox{\textwidth}{!}{%
  \begin{tabular}{@{}lccccccc@{}}
    \toprule
    & \multicolumn{3}{c}{Base geometry}
    & \multicolumn{2}{c}{Sharpness baselines}
    & \multicolumn{2}{c}{Combined geometry and sharpness} \\
    \cmidrule(lr){2-4}\cmidrule(lr){5-6}\cmidrule(l){7-8}
    Architecture
      & \textsc{SGDM}
      & \textcolor{LLQRSAMNGD}{\LLQR}
      & \textcolor{LLQRSAMNewton}{\LLQR}
      & \textsc{SAM}
      & \textsc{FisherSAM}
      & \textcolor{LLQRSAMNGD}{\LLQRSAM}
      & \textcolor{LLQRSAMNewton}{\LLQRSAM} \\
    \midrule
    VGG16-BN
      & $75.65 \pm 0.28$
      & $76.32 \pm 0.35$
      & $76.31 \pm 0.11$
      & $76.70 \pm 0.18$
      & $76.70 \pm 0.10$
      & \textcolor{LLQRSAMNGD}{$\textbf{78.14} \pm 0.27$}
      & \textcolor{LLQRSAMNewton}{$77.67 \pm 0.27$} \\
    ResNet-18
      & $79.27 \pm 0.18$
      & $79.53 \pm 0.39$
      & $79.83 \pm 0.18$
      & $81.13 \pm 0.21$
      & $80.89 \pm 0.31$
      & \textcolor{LLQRSAMNGD}{$\textbf{81.82} \pm 0.14$}
      & \textcolor{LLQRSAMNewton}{$81.65 \pm 0.10$} \\
    WRN-28-10
      & $82.50 \pm 0.18$
      & $82.80 \pm 0.37$
      & $82.72 \pm 0.24$
      & $84.83 \pm 0.06$
      & $84.71 \pm 0.13$
      & \textcolor{LLQRSAMNGD}{$\textbf{85.30} \pm 0.14$}
      & \textcolor{LLQRSAMNewton}{$85.05 \pm 0.12$} \\
    PyramidNet-110
      & $83.96 \pm 0.25$
      & $84.40 \pm 0.46$
      & $84.00 \pm 0.16$
      & $86.20 \pm 0.18$
      & $85.92 \pm 0.10$
      & \textcolor{LLQRSAMNGD}{$\textbf{86.67} \pm 0.21$}
      & \textcolor{LLQRSAMNewton}{$86.53 \pm 0.05$} \\
    \bottomrule
  \end{tabular}%
  }
\end{table}
Tables~\ref{tab:cifar10-standard-benchmarks} and
\ref{tab:cifar100-standard-benchmarks} show that \LLQRSAM consistently
improves over either component alone, indicating that coupling \SAM with the
learned \LLQR geometry is more effective than applying either mechanism in
isolation. Across CIFAR-10 and CIFAR-100, the NGD
variant delivers the strongest results, suggesting that the natural-gradient geometry provides the more effective steepest descent metric in those setups.

\providecommand{\TinyImageNetPlotDir}{tinyimagenet_plots}
\textbf{TinyImageNet.}
We next scale the ResNet family networks to TinyImageNet-200 to test whether the
geometry--sharpness coupling remains effective beyond CIFAR.
\begin{table}[t]
  \centering
  \caption{TinyImageNet-200 ResNet-family benchmark. Entries are best top-1 test
    accuracy (mean $\pm$ std, 3 seeds). \textcolor{LLQRSAMNGD}{Blue}
    denotes NGD-induced geometry, \textcolor{LLQRSAMNewton}{Orange} denotes
    Newton-induced geometry.}
  \label{tab:tinyimagenet-resnet-benchmarks}
  \small
  \resizebox{\textwidth}{!}{%
  \begin{tabular}{@{}lcccccc@{}}
    \toprule
    & \multicolumn{3}{c}{Base geometry}
    & \multicolumn{1}{c}{Sharpness baseline}
    & \multicolumn{2}{c}{Combined geometry and sharpness} \\
    \cmidrule(lr){2-4}\cmidrule(lr){5-5}\cmidrule(l){6-7}
    Architecture
      & \textsc{SGDM}
      & \textcolor{LLQRSAMNGD}{\LLQR}
      & \textcolor{LLQRSAMNewton}{\LLQR}
      & \textsc{SAM}
      & \textcolor{LLQRSAMNGD}{\LLQRSAM}
      & \textcolor{LLQRSAMNewton}{\LLQRSAM} \\
    \midrule
    ResNet-18
      & $57.97 \pm 0.28$
      & \textcolor{LLQRSAMNGD}{$57.02 \pm 0.49$}
      & \textcolor{LLQRSAMNewton}{$57.62 \pm 0.47$}
      & $61.17 \pm 0.27$
      & \textcolor{LLQRSAMNGD}{$\textbf{62.24} \pm 0.49$}
      & \textcolor{LLQRSAMNewton}{$61.98 \pm 0.22$} \\
    WRN-28-10
      & $61.06 \pm 0.37$
      & \textcolor{LLQRSAMNGD}{$60.08 \pm 0.32$}
      & \textcolor{LLQRSAMNewton}{$60.37 \pm 0.33$}
      & $63.91 \pm 0.41$
      & \textcolor{LLQRSAMNGD}{$\textbf{64.65} \pm 0.51$}
      & \textcolor{LLQRSAMNewton}{$64.59 \pm 0.25$} \\
    \bottomrule
  \end{tabular}%
  }
  \vspace{-1.0em}
\end{table}
Table~\ref{tab:tinyimagenet-resnet-benchmarks} shows that the CIFAR trend
persists when scaling to TinyImageNet. \SAM gives a strong gain over
\textsc{SGDM}, while \LLQR alone is not sufficient or not well-tuned in this setting. Yet, the
combined \LLQRSAM variants improve over \SAM on both architectures, with the
NGD-induced geometry giving the best result in each case.  This is a
particularly stringent test of complementarity: even when \LLQR alone does not
improve over \textsc{SGDM}, its geometry still strengthens \SAM, suggesting that
the learned geometry provides useful information that is not captured by
the sharpness perturbation alone.

\textbf{Synergy with \SAM variants.}
The \SAM framework has led to a broad family of perturb-and-recompute
optimizers, many differing only in how the adversarial direction is filtered or
normalized.  To check that \LLQR is not merely tuned to vanilla \SAM, we pair it
with \textsc{F-SAM}~\citep{li2024friendly}, a strong \SAM variant.  Performance
is compared with and without the learned \LLQR geometry, using vanilla \SAM as
a reference in Table~\ref{tab:cifar-sam-variant-synergy}.

\begin{table}[t]
  \centering
  \caption{CIFARs Friendly-SAM synergy benchmark. Entries are best top-1 test
    accuracy (mean $\pm$ std, 5 seeds). \textcolor{LLQRSAMNGD}{Blue}
    denotes NGD-induced \LLQR geometry.}
  \label{tab:cifar-sam-variant-synergy}
  \scriptsize
  \resizebox{\textwidth}{!}{%
  \begin{tabular}{@{}llcccc@{}}
    \toprule
    & & \multicolumn{2}{c}{Without \LLQR}
      & \multicolumn{2}{c}{With \LLQR geometry} \\
    \cmidrule(lr){3-4}\cmidrule(l){5-6}
    Dataset & Architecture
      & \textsc{SAM}
      & \textsc{F-SAM}
      & \textcolor{LLQRSAMNGD}{\LLQRSAM}
      & \textcolor{LLQRSAMNGD}{\LLQR+\textsc{F-SAM}} \\
    \midrule
    CIFAR-10 & VGG16-BN
      & $95.54 \pm 0.14$
      & $95.58 \pm 0.09$
      & \textcolor{LLQRSAMNGD}{$\textbf{95.92} \pm 0.07$}
      & \textcolor{LLQRSAMNGD}{$95.81 \pm 0.11$} \\
    CIFAR-10 & ResNet-18
      & $96.65 \pm 0.16$
      & $96.60 \pm 0.05$
      & \textcolor{LLQRSAMNGD}{$96.85 \pm 0.09$}
      & \textcolor{LLQRSAMNGD}{$\textbf{96.88} \pm 0.11$} \\
    CIFAR-10 & WRN-28-10
      & $97.43 \pm 0.09$
      & $97.49 \pm 0.08$
      & \textcolor{LLQRSAMNGD}{$97.64 \pm 0.08$}
      & \textcolor{LLQRSAMNGD}{$\textbf{97.68} \pm 0.07$} \\
    CIFAR-10 & PyramidNet-110
      & $97.63 \pm 0.09$
      & $97.65 \pm 0.07$
      & \textcolor{LLQRSAMNGD}{$97.86 \pm 0.04$}
      & \textcolor{LLQRSAMNGD}{$\textbf{97.89} \pm 0.06$} \\
    \midrule
    CIFAR-100 & VGG16-BN
      & $76.70 \pm 0.18$
      & $76.95 \pm 0.32$
      & \textcolor{LLQRSAMNGD}{$78.14 \pm 0.27$}
      & \textcolor{LLQRSAMNGD}{$\textbf{78.22} \pm 0.31$} \\
    CIFAR-100 & ResNet-18
      & $81.13 \pm 0.21$
      & $81.16 \pm 0.12$
      & \textcolor{LLQRSAMNGD}{$81.82 \pm 0.14$}
      & \textcolor{LLQRSAMNGD}{$\textbf{82.00} \pm 0.28$} \\
    CIFAR-100 & WRN-28-10
      & $84.83 \pm 0.06$
      & $84.86 \pm 0.16$
      & \textcolor{LLQRSAMNGD}{$\textbf{85.30} \pm 0.14$}
      & \textcolor{LLQRSAMNGD}{$85.21 \pm 0.20$} \\
    CIFAR-100 & PyramidNet-110
      & $86.20 \pm 0.18$
      & $86.23 \pm 0.08$
      & \textcolor{LLQRSAMNGD}{$86.67 \pm 0.21$}
      & \textcolor{LLQRSAMNGD}{$\textbf{86.76} \pm 0.13$} \\
    \bottomrule
  \end{tabular}%
  }
\end{table}

\begin{figure*}[b]
  \vspace{-1.0em}
  \centering
  \includegraphics[width=0.9\textwidth]{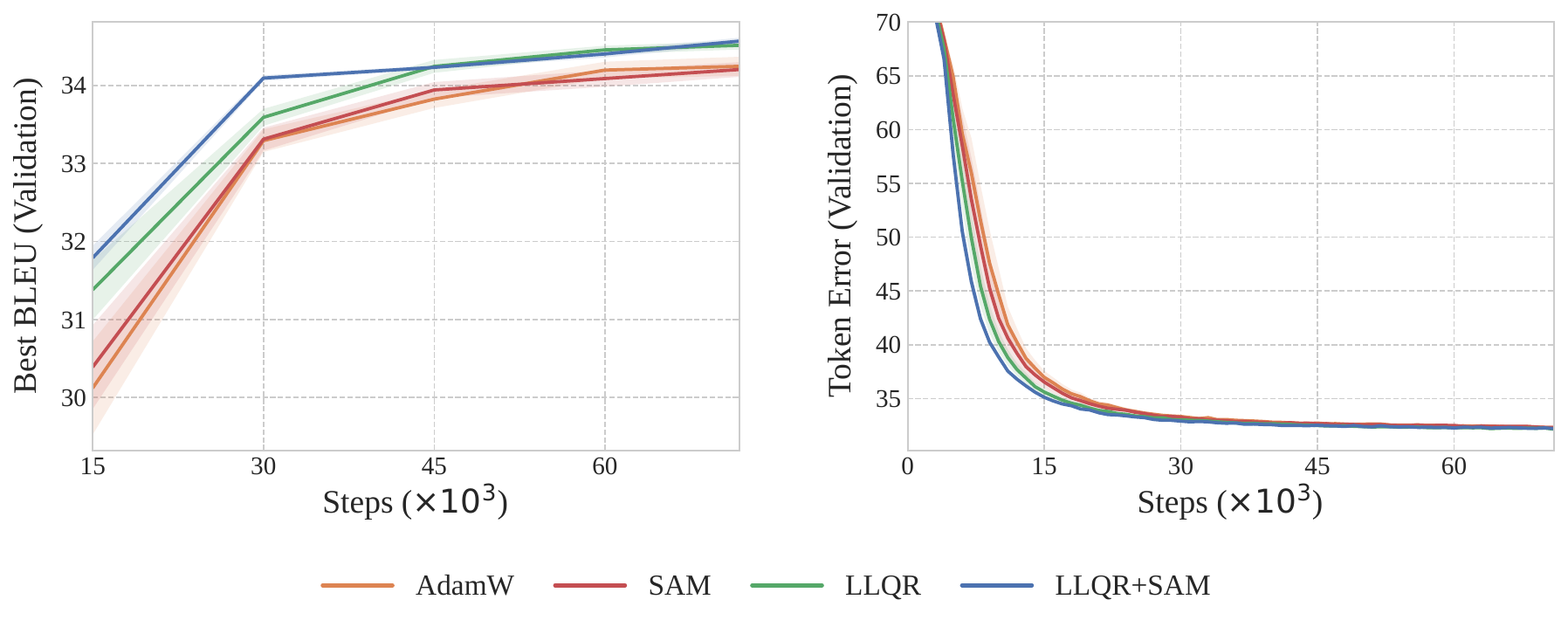}
  \caption{\textbf{IWSLT14 German-to-English convergence.}
  Validation BLEU and token error curves for the fairseq Transformer benchmark.
  Pairing \LLQR with \SAM accelerates optimization while offering the modest
  best-performance gains reported in Table~\ref{tab:iwslt14-base-sam-llqr-ngd}}
  \label{fig:iwslt14_de_en_curves}
\end{figure*}
\textbf{ImageNet.}
We also extend the comparison to ResNet-50 on
ImageNet~\citep{imagenetDeng2009}.  As shown in
Table~\ref{tab:imagenet-fsam-synergy}, adding \LLQR to \textsc{F-SAM}
again improves performance over the sharpness-aware baseline.  The full top-1
error trajectory in Fig.~\ref{fig:imagenet_fsam_top1_loss_time} (left) makes
the effect even clearer: the paired method not only reaches better
generalization, but also converges faster across early and mid training. These
results suggest that \LLQR and \SAM-family pairing synergize throughout the entire training process, even at ImageNet scale.

\begin{table}[t]
  \centering
  \caption{ImageNet/ResNet-50 benchmark. Entries are best top-1 test accuracy
    (mean $\pm$ std, 3 seeds). \textcolor{LLQRSAMNGD}{Blue} denotes NGD-induced
    geometry, \textcolor{LLQRSAMNewton}{Orange} denotes Newton-induced
    geometry. \LLQR rows use the E-KFAC block structure.}
  \label{tab:imagenet-fsam-synergy}
  \small
  \resizebox{\textwidth}{!}{%
  \begin{tabular}{@{}lcccccc@{}}
    \toprule
    & \multicolumn{3}{c}{Base geometry}
    & \multicolumn{1}{c}{Sharpness baseline}
    & \multicolumn{2}{c}{Combined geometry and sharpness} \\
    \cmidrule(lr){2-4}\cmidrule(lr){5-5}\cmidrule(l){6-7}
    Architecture
      & \textsc{SGDM}
      & \textcolor{LLQRSAMNGD}{\LLQR}
      & \textcolor{LLQRSAMNewton}{\LLQR}
      & \textsc{F-SAM}
      & \textcolor{LLQRSAMNGD}{\LLQR+\textsc{F-SAM}}
      & \textcolor{LLQRSAMNewton}{\LLQR+\textsc{F-SAM}} \\
    \midrule
    ResNet-50
      & $77.60 \pm 0.31$
      & \textcolor{LLQRSAMNGD}{$78.05 \pm 0.12$}
      & \textcolor{LLQRSAMNewton}{$76.93 \pm 0.42$}
      & $77.92 \pm 0.07$
      & \textcolor{LLQRSAMNGD}{$\textbf{78.61} \pm 0.16$}
      & \textcolor{LLQRSAMNewton}{$78.19 \pm 0.14$} \\
    \bottomrule
  \end{tabular}%
  }
\end{table}
\textbf{IWSLT14 German-to-English.}
To assess if the geometry--sharpness coupling transfers beyond vision, we
evaluate on the fairseq IWSLT14 German-to-English Transformer benchmark.  As
shown in Table~\ref{tab:iwslt14-base-sam-llqr-ngd}, the gains are
more modest than in image classification, but \LLQRSAM still achieves the best
BLEU. More strikingly, the convergence curves in Fig.~\ref{fig:iwslt14_de_en_curves}
show faster optimization when \LLQR and \SAM are paired, extending the same pattern
observed in vision to a sequence-modeling benchmark.

\begin{table}[t]
  \centering
  \caption{IWSLT14 German-to-English Transformer benchmark. Entries are
    best validation BLEU (mean $\pm$ std, 5 seeds). \textcolor{LLQRSAMNGD}{Blue}
    denotes NGD-induced \LLQR geometry.}
  \label{tab:iwslt14-base-sam-llqr-ngd}
  \small
  \resizebox{\textwidth}{!}{%
  \begin{tabular}{@{}lcccc@{}}
    \toprule
    & \multicolumn{2}{c}{Base geometry}
    & \multicolumn{1}{c}{Sharpness baseline}
    & \multicolumn{1}{c}{Combined geometry and sharpness} \\
    \cmidrule(lr){2-3}\cmidrule(lr){4-4}\cmidrule(l){5-5}
    Dataset
      & \textsc{AdamW}
      & \textcolor{LLQRSAMNGD}{\LLQR}
      & \textsc{SAM}
      & \textcolor{LLQRSAMNGD}{\LLQRSAM} \\
    \midrule
    IWSLT14 De-En
      & $34.24 \pm 0.27$
      & \textcolor{LLQRSAMNGD}{$34.51 \pm 0.12$}
      & $34.20 \pm 0.19$
      & \textcolor{LLQRSAMNGD}{$\textbf{34.57} \pm 0.09$} \\
    \bottomrule
  \end{tabular}%
  }
\end{table}

\subsection{Isolating the role of geometry in probe and transport: FisherSAM Ablation}
\begin{wraptable}{r}{0.35\textwidth}
  \vspace{-2.0em}
  \centering
  \caption{ResNet-18/CIFAR-100 FisherSAM-style perturbation ablation. Entries are best top-1 accuracy (mean $\pm$ std, five seeds). \textcolor{LLQRSAMNGD}{$\mathrm{LLQR}^{\Delta}$+\SAM} uses LLQR preconditioner only in the perturbation step, similar to FisherSAM.}
  \label{tab:fisher-sam-perturbation-only}
  \scriptsize
  \resizebox{0.35\textwidth}{!}{
  \begin{tabular}{@{}lc@{}}
    \toprule
    Method & Top-1 accuracy \\
    \midrule
    \textsc{FisherSAM} & $80.89 \pm 0.31$ \\
    \textcolor{LLQRSAMNGD}{$\mathrm{LLQR}^{\Delta}$+\SAM} & \textcolor{LLQRSAMNGD}{$81.23 \pm 0.14$} \\
    \textcolor{LLQRSAMNGD}{\LLQRSAM} & \textcolor{LLQRSAMNGD}{$\textbf{81.82} \pm 0.14$} \\
    \bottomrule
  \end{tabular}
  }
  \vspace{-3.5em}
\end{wraptable}

FisherSAM uses its geometry only to define the perturbation; LLQR+SAM additionally uses it to transport the outer update. To separate these two roles, we introduce
\(\mathrm{LLQR}^{\Delta}\)+\SAM, which applies the learned preconditioner only in the perturbation step (matching the FisherSAM pattern).  On ResNet-18/CIFAR-100 (Table~\ref{tab:fisher-sam-perturbation-only}),  \(\mathrm{LLQR}^{\Delta}\)+\SAM already improves over FisherSAM--the LLQR geometry is itself a better probe--and full \LLQRSAM improves further, showing that transporting the outer update in the same geometry adds a separate gain.

\subsection{Scalability}

To assess scalability, we evaluate the computational and memory footprint of \LLQR and \LLQRSAM on ViT models of increasing size~\citep{dosovitskiy2020image}, from
ViT-Tiny/16 to ViT-Large/16,  on
ImageNet-scale data. We compare LLQR and LLQR+SAM against SGD, AdamW, and SAM, measuring wall-clock time across update cadences.  Figure~\ref{fig:lqr-time-complexity} reports the
steady preconditioner-update time as a function of model size: empirically, the recurring LLQR cost scales as $\mathcal{O}(P^{1.27\text{--}1.28})$
on the number of parameters 
$P$, closer to linear than to the quadratic scaling typically associated with second-order methods. Pairing LLQR with SAM does not change this scaling. Memory is controlled by the choice of preconditioner block structure (E-KFAC, K-FAC, diag-KFAC, diagonal, etc.);
Appendix~\ref{appsec:scalability} details the storage tradeoffs and reports the cadence sweep across update intervals.
We also provide cadence results in Table~\ref{tab:vit-bl16-cadence-details} (Appendix~\ref{appsec:scalability}) and
Fig.~\ref{fig:appendix-vit-cadence}, to be read as wall-clock overhead
measurements, not as accuracy or convergence comparisons.  As expected, the
cost decreases when \LLQR updates are performed less frequently, since fewer
preconditioner refreshes are executed per epoch. This steady-update scaling trend is further exposed in
Fig.~\ref{fig:lqr-time-complexity}.

\begin{figure}[t]
  \centering
  \includegraphics[width=0.49\linewidth]{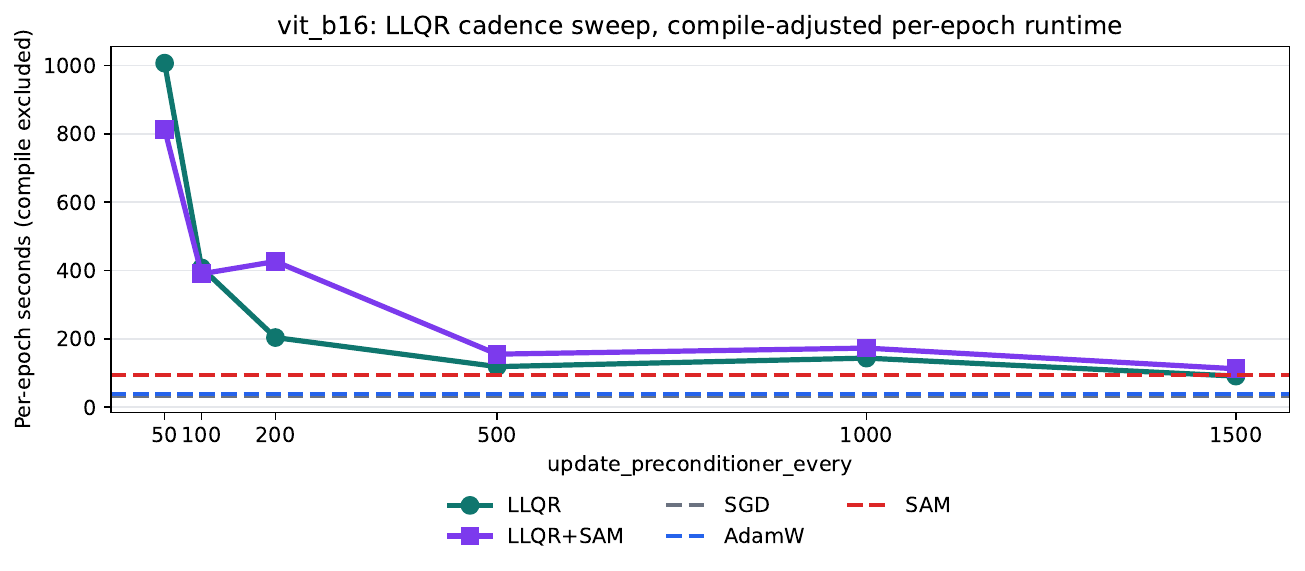}
  \includegraphics[width=0.49\linewidth]{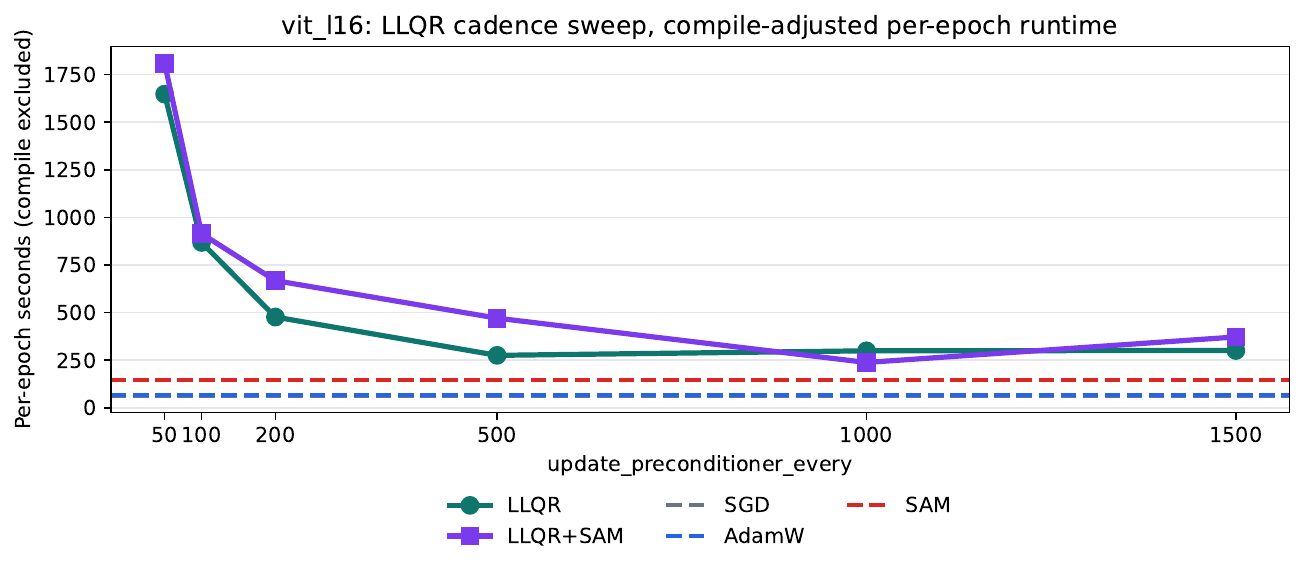}
  \caption{\LLQR{} cadence sweep for ViT-B/16 and ViT-L/16. Each panel shows
compile-adjusted seconds per epoch versus the \LLQR{} update interval, with
first-order baselines as horizontal references. Note that the ImageNet ResNet-50 runs in
Fig.~\ref{fig:imagenet_fsam_top1_loss_time} use \(1500\)-step updates.}
  \label{fig:appendix-vit-cadence}
\end{figure}


\section{Conclusion}

\begin{wrapfigure}{r}{0.50\linewidth}
  \centering
  \vspace{-2.5em}
  \includegraphics[width=1.0\linewidth]{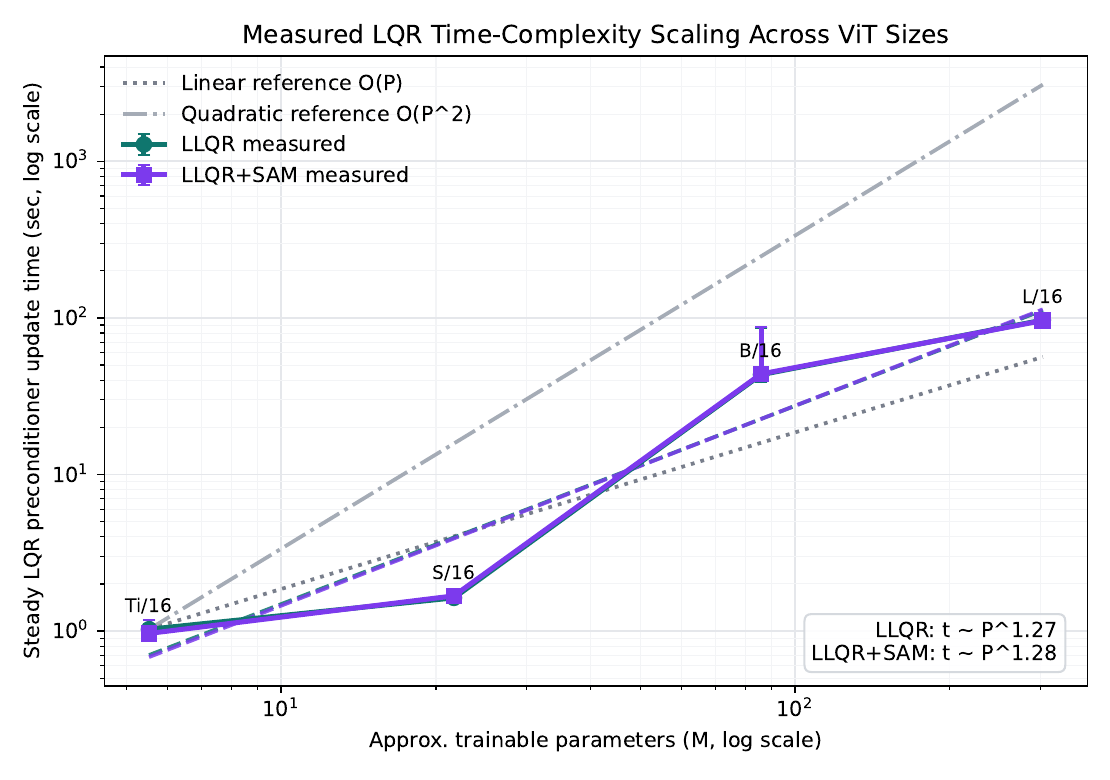}
  \vspace{-1.5em}
  \caption{Across ViT scales, preconditioner-update time grows nearly linearly for both
\LLQR and \LLQRSAM, far from
the quadratic scaling typically associated with second-order methods.  
}
  \label{fig:lqr-time-complexity}
  \vspace{-0.5em}
\end{wrapfigure}

We introduced LLQR+SAM, which pairs a slowly-updated LLQR preconditioner with a SAM perturbation evaluated and transported in the induced geometry. On a quadratic two-scale model the dynamics is closed-form: SAM prevents the iterate from localizing at any minimum, and the learned 
$U$ shapes the hovering scale to match the average geometry--large enough to escape sharp potholes, small enough to stay inside wide basins. The pothole-direction probe is amplified relative to vanilla SAM by a factor that grows as the surrounding basin becomes flatter--the regime where preconditioning matters most. The same $U$ that improves conditioning therefore also improves escape: 
$U$ and SAM are not independent ingredients but two complementary uses of one learned geometry.

Empirically, LLQR+SAM yields consistent gains over either component alone across CIFAR-10/100, TinyImageNet, ImageNet, and IWSLT14, including when paired with stronger SAM variants such as F-SAM, and the per-step overhead remains close to first-order training across ViT scales.

\paragraph{Limitations}
Our analysis relies on an idealized two-timescale quadratic model in which the slow \LLQR metric captures average
geometry and the \SAM perturbation probes localized sharpness.  This setting
is useful for exposing the mechanism, but it cannot certify that the same
decomposition transfers cleanly to the highly nonconvex and stochastic regime of
deep networks.

The empirical scope is also finite: we cover standard supervised vision
benchmarks and one sequence-modeling task, but not large language model pre-training or fine-tuning, nor every architecture family. 

Finally, \LLQRSAM inherits the extra gradient evaluation of \SAM and
adds periodic \LLQR metric updates;  while the resulting overhead is practical in our experiments, the best block structure, update cadence, and memory
budget remain scale-dependent design choices.

\clearpage
\bibliography{llqr_sam_bib}
\bibliographystyle{plainnat}

\clearpage
\appendix

\tableofcontents

\clearpage

\section{Layerwise LQR: Learning a Reusable Geometry}
\label{app:llqr-summary}

\begin{algorithm}[b]
\caption{\LLQR: relaxed periodic preconditioner}
\label{alg:llqr-summary}
\small
\begin{algorithmic}[1]
\REQUIRE model $f$, loss $L$, data loader $\mathcal D$, outer optimizer $\mathcal O_{\rm out}$, step size $\eta$
\REQUIRE recompute period $n$, inner steps $T$, structure $\mathcal S$, inner optimizer $\mathcal O_{\rm in}$, inner step size $\alpha$, EMA parameter $\beta$
\STATE Initialize $\theta^0$ and $U\leftarrow I$ projected to $\mathcal S$
\FOR{$k=0,1,2,\ldots$}
  \STATE Sample minibatch $(x^k,y^k)\sim\mathcal D$ and compute $g_k\leftarrow\nabla_\theta L(\theta^k)$
  \IF{$k \bmod n=0$}
    \STATE Linearize the network: $(A_i,B_i)\leftarrow(\nabla_x,\nabla_\theta)f_i(x_i^k,\theta_i^k)$
    \STATE Form the LQR blocks $(Q_i,R_i,M_i)$ for the chosen divergence-induced metric
    \STATE Set $U_0\leftarrow U$
    \FOR{$t=0$ to $T-1$}
      \STATE $U_{t+1}\leftarrow
      \mathcal O_{\rm in}\!\left(U_t,\nabla_{U_t}J,\alpha\right)$,
      where $J$ is the relaxed objective \eqref{eq:llqr-relaxed-summary}
    \ENDFOR
    \STATE $U\leftarrow \beta U+(1-\beta)U_T$
  \ENDIF
  \STATE Precondition the gradient: $\widetilde g_k\leftarrow Ug_k$
  \STATE Update parameters with the base optimizer:
  $\theta^{k+1}\leftarrow \mathcal O_{\rm out}(\theta^k,\widetilde g_k,\eta)$
\ENDFOR
\end{algorithmic}
\end{algorithm}

We briefly summarize the Layerwise LQR (\LLQR) methodology of
\citet{Anonymous2026LLQR}, which provides the geometry used throughout this work.
The motivation is to obtain the benefits of second-order or natural-gradient
descent without explicitly forming, storing, or inverting a global curvature
matrix.  At an iterate $\theta^k$, a broad class of geometry-aware descent
directions can be written as the solution of the local quadratic problem
\begin{equation}
\label{eq:llqr-steepest-summary}
\Delta\theta^\star
=
\arg\min_{\Delta\theta}
\left\{
\nabla L(\theta^k)^\top \Delta\theta
+
\frac12 \Delta\theta^\top H(\theta^k)\Delta\theta
\right\},
\end{equation}
where $H$ may be a regularized Hessian, Gauss--Newton matrix, Fisher matrix, or
more general divergence-induced metric.  Directly solving
$H\Delta\theta=-\nabla L(\theta^k)$ is infeasible for modern networks, and
standard scalable approximations typically impose block structure on $H$ before
solving the problem.  \LLQR instead first rewrites the dense quadratic problem
as a layerwise optimal-control problem, and only later imposes structure on the
learned inverse update.

Consider a depth-$N$ network
\[
x_{i+1}=f_i(x_i,\theta_i),\qquad i=0,\ldots,N-1.
\]
Linearizing the forward pass at $(x_i^k,\theta_i^k)$ gives the perturbation
dynamics
\begin{equation}
\label{eq:llqr-linearized-summary}
\delta x_{i+1}
=
A_i\delta x_i+B_i\delta\theta_i,
\qquad
\delta x_0=0,
\end{equation}
with $A_i=\partial f_i/\partial x_i$ and
$B_i=\partial f_i/\partial \theta_i$.  For divergence-induced quadratic
models, the global quadratic form in \eqref{eq:llqr-steepest-summary} can then
be decomposed into layerwise costs
\[
\delta x_N^\top Q_N\delta x_N
+
\sum_{i=0}^{N-1}
\begin{bmatrix}
\delta x_i\\
\delta\theta_i
\end{bmatrix}^{\!\top}
\begin{bmatrix}
Q_i & M_i^\top\\
M_i & R_i
\end{bmatrix}
\begin{bmatrix}
\delta x_i\\
\delta\theta_i
\end{bmatrix},
\]
subject to the linearized dynamics \eqref{eq:llqr-linearized-summary}.  The
resulting problem is a finite-horizon Linear Quadratic Regulator: the network
defines the dynamics, while the chosen descent geometry defines the quadratic
costs.  Its exact Riccati solution recovers the corresponding second-order
update, including Newton, Gauss--Newton, and natural-gradient variants.  In
large networks, however, the exact Riccati recursion remains too expensive, so
\LLQR uses it as a reference objective rather than as the deployed optimizer.

The scalable relaxation parameterizes the update by a structured inverse
preconditioner
\[
\delta\theta_i=-U_i\nabla_{\theta_i}L(\theta^k),
\qquad
U=\operatorname{diag}(U_0,\ldots,U_{N-1}),
\]
where each $U_i$ may be diagonal, Kronecker-factored, E-KFAC-like, or another
structured block.  Substituting this restricted update into the LQR objective
yields a direct objective over $U$:
\begin{equation}
\label{eq:llqr-relaxed-summary}
\begin{aligned}
\min_U\quad
& \nabla_{x_N}\ell(x_N)^\top \delta x_N
+\frac12 \delta x_N^\top Q_N\delta x_N
+\sum_{i=0}^{N-1}
\left[
\frac12 \delta x_i^\top Q_i\delta x_i
+\frac12 \delta\theta_i^\top R_i\delta\theta_i
+\delta\theta_i^\top M_i\delta x_i
\right] \\
\text{s.t.}\quad
& \delta x_{i+1}
=
A_i\delta x_i
-
B_iU_i\nabla_{\theta_i}L(\theta^k),
\qquad
\delta x_0=0 .
\end{aligned}
\end{equation}
Thus, structure is imposed on the reusable inverse action $U$, not on the
curvature model before the layerwise objective is derived.  This distinction is
important for \LLQR+\SAM: the learned geometry still comes from a layer-coupled
objective that encodes the dense quadratic model through the forward dynamics,
while the deployed update only requires applying $U$ to gradients.

In the \LLQR+\SAM setting, $U$ should be viewed as a slowly updated transport
geometry.  Between preconditioner refreshes, it is held fixed and applied
cheaply to the current gradient.  This makes the method compatible with
standard optimizers and with perturb-and-recompute schemes: \SAM determines a
local sharpness-aware gradient, while \LLQR supplies the anisotropic geometry
used to probe and/or transport that gradient.

\clearpage
\section{Detailed Analysis for the Two-Scale Pothole Model}

This appendix contains the supporting arguments used in
Sections~\ref{sec:method} and~\ref{sec:analysis}. The first subsection records
the fixed-metric Riemannian-\SAM\ transfer for \LLQR+\SAM. The remaining
subsections justify the two-scale pothole calculation: the coordinate
recursion, the scalar hovering envelope, the vanilla-\SAM\ comparison, and the
stochastic selection picture.

\subsection{Riemannian-\SAM\ transfer under the frozen \LLQR\ metric}
\label{app:proof_riemannian_transfer}

This subsection records the convergence-transfer statement used in
Section~\ref{sec:method}. The statement is conditional on the \LLQR\ metric
being frozen during each perturb-and-recompute step; it gives stationarity in
the instantaneous metric, not in a fixed limiting metric unless additional
metric convergence is assumed.

Let \(P_t=U_t^{-1}\). Assume that \(U_t\) is measurable with respect to the
optimization history \(\mathcal F_t\), is fixed during the \SAM\
perturb-and-recompute step, and satisfies the uniform spectral bounds
\[
    0<u_{\min}I\preceq U_t\preceq u_{\max}I<\infty .
\]
Assume also that the smoothness, retraction-Lipschitz, and stochastic-gradient
assumptions of Riemannian \SAM~\citep{yun2023riemanniansam} hold uniformly for
the compact family of metrics \(\{P_t\}\).

\begin{proposition}[Riemannian-\SAM\ transfer for \LLQR+\SAM]
\label{prop:llqrsam_riemannian_transfer}
With the same decaying stepsize and perturbation radius schedule as in
Riemannian \SAM, for \(\tau\sim\operatorname{Unif}\{0,\ldots,T-1\}\),
\[
\mathbb E
\left[
\|\operatorname{grad}_{P_\tau}L(\theta_\tau)\|_{P_\tau}^2
\right]
=
\mathbb E
\left[
\nabla L(\theta_\tau)^\top U_\tau\nabla L(\theta_\tau)
\right]
=
O(T^{-1/2}),
\]
up to the same stochastic-gradient and minibatch-variance terms as in the
Riemannian-\SAM\ theorem. Consequently,
\[
    \mathbb E\|\nabla L(\theta_\tau)\|_2^2=O(T^{-1/2}).
\]
\end{proposition}

\begin{proof}
Condition on \(\mathcal F_t\). Since \(U_t\) is \(\mathcal F_t\)-measurable
and fixed during the perturb-and-recompute step, \(P_t=U_t^{-1}\) is a fixed
Riemannian metric for the current update. For the constant metric
\[
    \langle \xi,\zeta\rangle_{P_t}=\xi^\top P_t\zeta,
\]
the Riemannian gradient is
\[
    \operatorname{grad}_{P_t}L(\theta)
    =
    P_t^{-1}\nabla L(\theta)
    =
    U_t\nabla L(\theta),
\]
and hence
\[
\|\operatorname{grad}_{P_t}L(\theta)\|_{P_t}^2
=
\nabla L(\theta)^\top U_t\nabla L(\theta).
\]
The \(\LLQR+\SAM\) perturbation
\[
    \epsilon_t
    =
    \rho_t\frac{U_tg_t}{(g_t^\top U_tg_t)^{1/2}}
\]
is therefore the normalized Riemannian steepest-ascent direction under the
frozen metric \(P_t\). The update
\[
    \theta_{t+1}
    =
    \theta_t-\alpha_tU_t\nabla L_{b_t}(\theta_t+\epsilon_t)
\]
is the corresponding Riemannian steepest-descent step on the probed loss.
Conditioned on \(\mathcal F_t\), this is one fixed-metric Riemannian-\SAM\ step.

The spectral bounds on \(U_t\) place all metrics \(P_t=U_t^{-1}\) in a compact
subset of the positive-definite cone. Hence the norm-equivalence, smoothness,
retraction-Lipschitz, and stochastic-gradient constants in the Riemannian-\SAM\
one-step descent inequality can be chosen uniformly over \(t\). Applying that
inequality conditionally on \(\mathcal F_t\), then taking expectation and
summing over \(t=0,\ldots,T-1\), gives
\[
\frac1T\sum_{t=0}^{T-1}
\mathbb E
\|\operatorname{grad}_{P_t}L(\theta_t)\|_{P_t}^2
=
O(T^{-1/2}),
\]
up to the same stochastic-gradient and minibatch-variance terms. Equivalently,
for \(\tau\sim\operatorname{Unif}\{0,\ldots,T-1\}\),
\[
\mathbb E
\|\operatorname{grad}_{P_\tau}L(\theta_\tau)\|_{P_\tau}^2
=
O(T^{-1/2}).
\]

Finally, since \(U_t\succeq u_{\min}I\),
\[
\nabla L(\theta_t)^\top U_t\nabla L(\theta_t)
\ge
u_{\min}\|\nabla L(\theta_t)\|_2^2,
\]
so stationarity in the instantaneous \(\LLQR\) metric implies Euclidean
first-order stationarity:
\[
\mathbb E\|\nabla L(\theta_\tau)\|_2^2
\le
u_{\min}^{-1}
\mathbb E
\left[
\nabla L(\theta_\tau)^\top U_\tau\nabla L(\theta_\tau)
\right]
=
O(T^{-1/2}).
\]

If \(U_t\to U_f\) and \(\|\nabla L(\theta_t)\|_2\le G\), then
\[
\left|
\nabla L(\theta_t)^\top U_t\nabla L(\theta_t)
-
\nabla L(\theta_t)^\top U_f\nabla L(\theta_t)
\right|
\le
G^2\|U_t-U_f\|_{\mathrm{op}}.
\]
Averaging gives
\[
\frac1T\sum_{t=0}^{T-1}
\mathbb E
\left[
\nabla L(\theta_t)^\top U_f\nabla L(\theta_t)
\right]
\le
O(T^{-1/2})
+
G^2\frac1T\sum_{t=0}^{T-1}
\mathbb E\|U_t-U_f\|_{\mathrm{op}}.
\]
Thus convergence in the limiting metric follows whenever this average
discrepancy vanishes, and the original rate is preserved when the discrepancy
is \(O(T^{-1/2})\).
\end{proof}

\subsection{Whitened form and non-commuting case}
\label{app:whitened_pothole_model}
The commuting hypothesis used in the main text is not essential. Define
\[
    y_t := \bar H^{1/2}e_t,
    \qquad
    A := \bar H^{-1/2}H\bar H^{-1/2}
    =
    I+\bar H^{-1/2}H_\epsilon\bar H^{-1/2}.
\]
Since \(A\succ0\) is symmetric, it admits an orthonormal eigendecomposition.
Its eigenvalues are the perceived sharpnesses \(\mu_i\). In the commuting
case, these reduce to
\[
    \mu_i
    =
    \frac{\lambda_i}{\bar\lambda_i}
    =
    1+\frac{\lambda_{\epsilon,i}}{\bar\lambda_i}.
\]
Using \(U=\bar H^{-1}\), the exact LLQR+SAM recursion becomes, in whitened
coordinates,
\begin{equation}
    y_{t+1}
    =
    (I-\eta A)y_t
    -
    \frac{\eta\rho}{\|Ay_t\|_2}A^2y_t.
    \label{app:eq:whitened-dyn}
\end{equation}
Thus the non-commuting case is identical after diagonalizing the normalized
curvature matrix \(A\). The shared-eigenbasis presentation in the main text is
only a notational simplification.

\subsection{Coordinate recursion}
\label{sec:coordinate_recursion}

For the quadratic loss \(L(\theta)=\tfrac12\theta^\top H\theta\), we have
\(g(\theta)=H\theta\). The LLQR+SAM perturbation gives
\[
    g(\theta^+)
    =
    H\theta
    +
    \rho\frac{HUH\theta}{\|H\theta\|_U},
\]
and therefore, with \(e_t=\theta_t\),
\begin{equation}
    e_{t+1}
    =
    (I-\eta UH)e_t
    -
    \frac{\eta\rho}{\|He_t\|_U}UHUH e_t .
    \label{app:eq:exact-dyn}
\end{equation}

In the commuting case, let \(v_i\) be the common eigenvectors of \(\bar H\),
\(H_\epsilon\), and \(H\), and write
\(e_t=\sum_i z_{i,t}v_i\). Since
\[
    \bar H v_i=\bar\lambda_i v_i,
    \qquad
    H v_i=\lambda_i v_i,
\]
we get
\[
    UH v_i
    =
    \bar H^{-1}Hv_i
    =
    \frac{\lambda_i}{\bar\lambda_i}v_i
    =
    \mu_i v_i,
    \qquad
    UHUH v_i=\mu_i^2v_i.
\]
Moreover,
\[
    \|He_t\|_U^2
    =
    (He_t)^\top U(He_t)
    =
    \sum_j \lambda_j^2\bar\lambda_j^{-1}z_{j,t}^2
    =
    \sum_j \mu_j^2\bar\lambda_j z_{j,t}^2.
\]
Substituting these identities into~\eqref{app:eq:exact-dyn} gives the full
coordinate recursion
\begin{equation}
    z_{i,t+1}
    =
    (1-\eta\mu_i)z_{i,t}
    -
    \frac{\eta\rho\mu_i^2}
    {\sqrt{\sum_j\mu_j^2\bar\lambda_j z_{j,t}^2}}
    z_{i,t}.
    \label{app:eq:coupled-dyn}
\end{equation}
If one eigendirection \(v_i\) dominates, the denominator reduces to
\(\mu_i\sqrt{\bar\lambda_i}|z_{i,t}|\), yielding the scalar recursion
\begin{equation}
    z_{i,t+1}
    =
    (1-\eta\mu_i)z_{i,t}
    -
    \frac{\eta\rho\mu_i}{\sqrt{\bar\lambda_i}}
    \operatorname{sign}(z_{i,t}).
    \label{app:eq:scalar-dyn}
\end{equation}
This is the one-dimensional model used in the main text.

\subsection{Scalar Hovering Envelope}
\label{sec:hovering_env}

Consider the scalar map
\begin{equation}
    z_{t+1}
    =
    a z_t-b\operatorname{sign}(z_t),
    \qquad
    a\in(0,1),
    \quad
    b>0.
    \label{app:eq:scalar-map}
\end{equation}
It has no nonzero fixed point. Indeed, if \(z>0\), then
\[
    z=az-b
    \quad\Longrightarrow\quad
    z=-\frac{b}{1-a}<0,
\]
which contradicts \(z>0\); the negative half-line gives the symmetric
contradiction. Hence the map cannot localize at a nonzero point and instead
oscillates around the origin.

Once the signs alternate, the limiting two-cycle has amplitude \(r_\star\)
satisfying
\[
    r_\star=b-ar_\star,
\]
so
\begin{equation}
    r_\star
    =
    \frac{b}{1+a}.
    \label{app:eq:two-cycle-generic}
\end{equation}
For~\eqref{app:eq:scalar-dyn},
\[
    a_i=1-\eta\mu_i,
    \qquad
    b_i=\frac{\eta\rho\mu_i}{\sqrt{\bar\lambda_i}},
\]
and therefore the exact asymptotic two-cycle amplitude is
\begin{equation}
    \limsup_{t\to\infty}|z_{i,t}|
    =
    \frac{b_i}{1+a_i}
    =
    \frac{\eta\rho\mu_i}
    {(2-\eta\mu_i)\sqrt{\bar\lambda_i}}.
    \label{app:eq:exact-two-cycle}
\end{equation}

The main text uses the simpler hovering envelope
\begin{equation}
    |z_{i,t}|
    \lesssim
    \frac{b_i}{1-a_i}
    =
    \frac{\rho}{\sqrt{\bar\lambda_i}}.
    \label{app:eq:hovering-envelope}
\end{equation}
This envelope is the scale relevant for the cancellation argument: the SAM kick
\(b_i\) and the contraction scale \(1-a_i=\eta\mu_i\) both grow linearly with
the perceived sharpness \(\mu_i\), leaving a radius controlled only by the
average curvature \(\bar\lambda_i\).

\subsection{Pothole Escape Criterion}
\label{app:pothole_escape_details}

The quadratic model is local, so it does not by itself describe the global
motion after leaving a well. It gives a localization criterion. If a pothole
has basin radius \(r_\epsilon\) in the relevant eigendirection, then the local
model predicts loss of localization when the hovering envelope exceeds that
radius:
\begin{equation}
    \frac{\rho}{\sqrt{\bar\lambda_\epsilon}}
    >
    r_\epsilon.
    \label{app:eq:escape-envelope}
\end{equation}
The key point is that this criterion depends on the average curvature
\(\bar\lambda_\epsilon\), not on the localized pothole sharpness
\(\lambda_\epsilon\). A sharper localized component increases the perceived
sharpness \(\mu_\epsilon\), but the SAM kick and contraction scale increase
together, producing the cancellation in~\eqref{app:eq:hovering-envelope}.

In a wide basin, small \(\bar\lambda_i\) is paired with a large basin radius,
so the same envelope can remain inside the basin. In a pothole, the basin
radius is small while the average direction remains flat, so the same envelope
can exceed the local basin radius.

\subsection{Vanilla SAM Comparison}
\label{app:vanilla_sam_comparison}

For vanilla SAM, \(U=I\). In an eigendirection of \(H\), the scalar recursion is
\begin{equation}
    z_{t+1}
    =
    (1-\eta\lambda_i)z_t
    -
    \eta\rho\lambda_i\operatorname{sign}(z_t).
    \label{app:eq:vanilla-sam-scalar}
\end{equation}
The same envelope calculation gives
\begin{equation}
    \frac{\eta\rho\lambda_i}{\eta\lambda_i}
    =
    \rho.
    \label{app:eq:vanilla-envelope}
\end{equation}
Thus vanilla SAM has a Euclidean hovering envelope of order \(\rho\) in every
direction, whereas LLQR+SAM has the direction-dependent envelope
\(\rho/\sqrt{\bar\lambda_i}\).

\begin{corollary}[Pothole-escape amplification]
\label{cor:amplification_app}
Under Assumption~\ref{assum:U} and the commuting hypothesis, the LLQR+SAM
hovering envelope around a pothole minimum is
\[
    \frac{\rho}{\sqrt{\bar\lambda_\epsilon}}
    =
    \frac{1}{\sqrt{\bar\lambda_\epsilon}}
    \times
    \bigl(\text{vanilla-SAM hovering envelope}\bigr).
\]
The amplification depends only on the average curvature
\(\bar\lambda_\epsilon\) in the pothole direction, not on the localized
pothole sharpness \(\lambda_\epsilon\).
\end{corollary}

This amplification cannot be reproduced selectively by retuning vanilla SAM.
Choosing
\[
    \rho^{\textsc{sam}}
    =
    \frac{\rho}{\sqrt{\bar\lambda_\epsilon}}
\]
would match the pothole-direction envelope, but it would enlarge the probe in
all directions, including non-pothole directions where this extra scale is not
needed. The contraction behavior also differs: vanilla SAM contracts at rate
\(1-\eta\lambda_i\), so its stable step size is constrained by the largest
curvature of \(H\). LLQR+SAM contracts at rate \(1-\eta\mu_i\), which removes
the conditioning of the average geometry \(\bar H\) from this constraint.

\subsection{Stochastic Selection Between Wells}
\label{app:selection}

The deterministic calculation above gives a local non-localization criterion:
when the hovering envelope is comparable to the radius of a sharp well, the
\SAM\ kick prevents long residence near its bottom. With stochastic gradients,
this becomes a selection effect. Noise may inject the iterate into a sharp
well, but the \SAM-induced hovering ejects it quickly; wide flat wells support
longer visits because their basin radius is large relative to the hovering
scale.

The following regenerative model formalizes this interpretation. It is not
intended as a full global model of training dynamics; it isolates the residence-time consequence of the local escape criterion.

\begin{proposition}[Regenerative selection between wells]
\label{prop:app_multiwell_selection}
Consider a regenerative idealization with wells indexed by $m$. At each cycle,
a well $J_n$ is sampled from a distribution $\nu$ with positive mass on each
well; the process starts near the well center, evolves under local
gradient-noise dynamics with noise scale $\sigma$, and exits when
$\|e_t\|\ge R$. Suppose flat wells satisfy
\[
    \mathbb E[\tau_R^{(m)}]
    \ge
    c_m^{\mathrm{flat}} R^2\sigma^{-2},
\]
whereas sharp unstable wells satisfy
\[
    \mathbb E[\tau_R^{(m)}]
    \le
    C_m^{\mathrm{sharp}}\bigl(1+\log(R/\sigma)\bigr).
\]
If at least one flat well is present, then the long-run occupation mass of
sharp wells satisfies
\[
    \sum_{m\in\mathcal M_{\mathrm{sharp}}}\mu_m(\sigma)
    =
    O\!\left(\frac{\sigma^2\log(R/\sigma)}{R^2}\right)
    \xrightarrow[\sigma\to0]{}0.
\]
\end{proposition}

\begin{proof}
By the renewal-reward formula,
\[
    \mu_m(\sigma)
    =
    \frac{\nu_m\mathbb E[\tau_R^{(m)}]}
    {\sum_\ell \nu_\ell\mathbb E[\tau_R^{(\ell)}]}.
\]
The total sharp-well numerator is
$O(1+\log(R/\sigma))$, while the denominator is at least the contribution of
one flat well, namely $\Omega(R^2\sigma^{-2})$. Dividing gives the stated
bound.
\end{proof}

For a pothole direction, the relevant dimensionless escape ratio is
\[
    \frac{\rho}{r_\epsilon\sqrt{\bar\lambda_\epsilon}}.
\]
Corollary~\ref{cor:amplification_app} increases this ratio for \LLQR+\SAM\
relative to vanilla \SAM, predicting shorter visits to sharp wells and faster
selection of the wide flat basin.

The next calculation supports the path-length behavior observed in the noisy
toy experiment. It states that, in a stable scalar noisy mode, increasing the
effective metric denominator damps both stationary variance and one-step
motion.

\begin{proposition}[Metric damping under gradient noise]
\label{prop:app_metric_noise}
Consider the scalar noisy recursion
\[
    z_{t+1}^{B}
    =
    q_i^B z_t^B-\eta(d_i^B)^{-1}\varepsilon_{i,t},
    \qquad
    q_i^B:=1-\eta\lambda_i(d_i^B)^{-1},
    \qquad
    \operatorname{Var}(\varepsilon_{i,t})=\tau_i^2,
\]
where $B$ denotes a metric choice and $d_i^B>\eta\lambda_i/2$. Then
\[
    \operatorname{Var}_{\pi_B}(z_i)
    =
    \frac{\eta\tau_i^2}{\lambda_i(2d_i^B-\eta\lambda_i)},
    \qquad
    \mathbb E_{\pi_B}\!\left[(z_{i,t+1}^B-z_{i,t}^B)^2\right]
    =
    \frac{2\eta^2\tau_i^2}{d_i^B(2d_i^B-\eta\lambda_i)}.
\]
Consequently, if $d_i^P\ge d_i^Q>\eta\lambda_i/2$, then metric $P$ has no
larger stationary variance or one-step motion than metric $Q$, with strict
improvement when $\tau_i^2>0$ and $d_i^P>d_i^Q$.
\end{proposition}

\begin{proof}
The AR(1) variance equation gives
\[
    V_i^B
    =
    (q_i^B)^2V_i^B+\frac{\eta^2\tau_i^2}{(d_i^B)^2}.
\]
Since
\[
    1-(q_i^B)^2
    =
    \frac{\eta\lambda_i(2d_i^B-\eta\lambda_i)}{(d_i^B)^2},
\]
we obtain
\[
    V_i^B
    =
    \frac{\eta\tau_i^2}{\lambda_i(2d_i^B-\eta\lambda_i)}.
\]
Moreover,
\[
    z_{i,t+1}^{B}-z_{i,t}^{B}
    =
    -\eta\lambda_i(d_i^B)^{-1}z_{i,t}^{B}
    -\eta(d_i^B)^{-1}\varepsilon_{i,t}.
\]
At stationarity, $z_{i,t}^{B}$ is independent of the fresh noise, so the cross
term vanishes and
\[
    \mathbb E_{\pi_B}\!\left[(z_{i,t+1}^{B}-z_{i,t}^{B})^2\right]
    =
    \frac{\eta^2}{(d_i^B)^2}\left(\lambda_i^2V_i^B+\tau_i^2\right)
    =
    \frac{2\eta^2\tau_i^2}{d_i^B(2d_i^B-\eta\lambda_i)}.
\]
Both denominators are increasing for $d_i^B>\eta\lambda_i/2$, which proves the
comparison.
\end{proof}

\clearpage
\section{Experiments Details}
\label{app:experiments-details}

\subsection{Standard Supervised Training Benchmarks}
\label{app:standard-supervised-training-details}

\paragraph{CIFAR protocol.}
For CIFAR-10 and CIFAR-100, all methods use the same supervised classification
protocol: a 200-epoch cosine schedule with initial learning rate \(0.05\),
Polyak momentum \(0.9\), batch size \(128\), random crop and horizontal flip
augmentation, normalization, and Cutout regularization~\citep{devries2017improved}.
Weight decay is selected in \(\{10^{-4},5\times10^{-4},10^{-3}\}\) based on
baseline best performance.  \SAM uses the recommended perturbation radius
\(\rho=0.1\) on CIFAR-10 and \(\rho=0.2\) on
CIFAR-100~\citep{foret2021sam,li2023enhancing,mi2022make,li2024friendly};
FisherSAM uses its recommended \(\rho=0.1\) and inverse-Fisher regularization
\(\eta=0.1\)~\citep{kim2022fishersam}.  The \LLQR rows keep the outer SGDM
recipe fixed and use the recommended inner preconditioner-learning settings of
\citet{Anonymous2026LLQR}: inner batch size \(128\), preconditioner learning rate
\(10^{-3}\), \(50\) inner momentum-solver steps, and update period \(500\)
optimizer steps.  The reported NGD rows use the NGD-induced divergence with EMA
\(0.95\), while the Newton rows use the Newton-induced divergence with EMA
\(0.9\).  \LLQRSAM reuses the same hyperparameters as its \LLQR and \SAM
components, without additional method-specific tuning.  Each CIFAR entry is the
best test accuracy observed during training, aggregated over five independent
seeds. All CIFAR experiments were run on NVIDIA L40S GPUs.

\paragraph{TinyImageNet protocol.}
TinyImageNet-200 uses the same training and reporting protocol and optimizer
hyperparameters as the CIFAR benchmarks, with the comparison restricted in the
main text to ResNet-18 and WRN-28-10.  Entries in
Table~\ref{tab:tinyimagenet-resnet-benchmarks} are best top-1 test accuracy
aggregated over three seeds. All TinyImageNet experiments were run on NVIDIA L40S GPUs.

\paragraph{\SAM-variant synergy protocol.}
The CIFAR \SAM-variant comparison pairs \LLQR with
\textsc{F-SAM}~\citep{li2024friendly}.  \textsc{F-SAM} uses the recommended
exponential moving average of gradient accumulation:
\(\lambda=0.9\) for WRN-28-10 and PyramidNet-110, and \(\lambda=0.6\) for the
other architectures.  The comparison otherwise follows the CIFAR protocol above
and reports best top-1 test accuracy over five seeds.
All experiments were run on NVIDIA L40S GPUs.

\paragraph{ImageNet protocol.}
The ImageNet comparison uses ResNet-50 on ImageNet~\citep{imagenetDeng2009}.
Relative to the preceding vision experiments, the only hyperparameter changes
are \(\rho=0.075\) and \(\lambda=0.95\), following the recommendation of
\citet{li2024friendly}.  The \LLQR rows use the E-KFAC block structure.
All ImageNet experiments were run on NVIDIA A100 GPUs.

\paragraph{IWSLT14 German-to-English protocol.}
The sequence-modeling benchmark uses the fairseq IWSLT14 German-to-English
Transformer training recipe of \citet{DBLP:conf/wmt/OttEGA18}.  The \SAM
perturbation radius is tuned over \([0.0005,0.2]\), with \(\rho=0.005\) giving
the best validation result. \SAM pertubation only starts being applied after learning rate warmup phase for stability. \LLQR hyperparameters are kept unchanged from the
vision experiments except for the preconditioner EMA decay, which is set to
\(0.925\).  Table~\ref{tab:iwslt14-base-sam-llqr-ngd} reports best
validation BLEU aggregated over five seeds. All IWSLT14 experiments were run on NVIDIA L40S GPUs.

\clearpage

\section{Additional Experimental Results}
\label{app:add_results}

\subsection{Noise injection robustness}
\providecommand{\NoiseInjectionPlotDir}{noise_injection_report}
\label{sec:noise-injection}

Since \SAM is known to improve robustness under label noise, we ask whether
learned geometry strengthens this effect.  On CIFAR-100, we corrupt only the
training labels with symmetric random flips at rate \(\gamma\), keep the test
set clean, and reuse the standard supervised training recipe. Table~\ref{tab:noise-injection-numbers}
reports the best clean-test accuracy across
\(\gamma \in \{0.2,0.6,0.7,0.8\}\). \LLQRSAM improves over \SAM at every
corruption level, with especially clear gains at moderate-to-high noise,
showing that transporting the sharpness-aware update through the learned
\LLQR geometry strengthens robustness under corrupted supervision.  The gains
remain positive even at \(\gamma=0.8\), where supervision is nearly
degenerate, the learned-geometry sharpness correction remains beneficial.

\begin{table}[t]
  \centering
  \caption{%
    Robustness to symmetric training-label corruption on CIFAR-100. Entries are
  best test accuracy on clean test labels (mean $\pm$ std, 3 seeds).
  The corruption rate \(\gamma\) denotes the fraction of randomly flipped
  training labels. \textcolor{LLQRSAMNGD}{Blue} denotes NGD-induced \LLQR
  geometry.
  }
  \label{tab:noise-injection-numbers}
  \small
  \resizebox{\textwidth}{!}{%
  \begin{tabular}{@{}llccccc@{}}
    \toprule
    & & \multicolumn{3}{c}{Base geometry}
      & \multicolumn{1}{c}{Sharpness baseline}
      & \multicolumn{1}{c}{Combined geometry and sharpness} \\
    \cmidrule(lr){3-5}\cmidrule(lr){6-6}\cmidrule(l){7-7}
    Dataset & $\gamma$
      & \textsc{SGDM}
      & \textsc{Adam}
      & \textcolor{LLQRSAMNGD}{\LLQR}
      & \textsc{SAM}
      & \textcolor{LLQRSAMNGD}{\LLQRSAM} \\
    \midrule
    CIFAR-100 & $0.2$
      & $63.41 \pm 0.46$
      & $41.01 \pm 0.63$
      & \textcolor{LLQRSAMNGD}{$64.32 \pm 0.22$}
      & $67.15 \pm 0.43$
      & \textcolor{LLQRSAMNGD}{$\textbf{68.64} \pm 0.36$} \\
    CIFAR-100 & $0.6$
      & $41.21 \pm 1.27$
      & $9.76 \pm 1.46$
      & \textcolor{LLQRSAMNGD}{$41.21 \pm 1.13$}
      & $46.81 \pm 1.46$
      & \textcolor{LLQRSAMNGD}{$\textbf{48.54} \pm 0.81$} \\
    CIFAR-100 & $0.7$
      & $32.52 \pm 0.72$
      & $4.16 \pm 0.52$
      & \textcolor{LLQRSAMNGD}{$34.51 \pm 0.71$}
      & $38.96 \pm 0.72$
      & \textcolor{LLQRSAMNGD}{$\textbf{43.27} \pm 0.82$} \\
    CIFAR-100 & $0.8$
      & $22.97 \pm 0.86$
      & $2.23 \pm 0.12$
      & \textcolor{LLQRSAMNGD}{$24.54 \pm 0.46$}
      & $31.19 \pm 1.74$
      & \textcolor{LLQRSAMNGD}{$\textbf{33.06} \pm 1.22$} \\
    \bottomrule
  \end{tabular}%
  }
\end{table}

\subsection{Scalability}
\label{appsec:scalability}

Our \LLQR{} implementation attaches a blockwise preconditioner to network
layers.  The storage footprint is therefore determined by the chosen
parameterization of each block rather than by \LLQR{} as a framework.  Richer
blocks can encode more geometry but require more optimizer state, while simpler
parameterizations reduce memory at the cost of expressivity.  Table~\ref{tab:model-scale-storage}
summarizes the resulting fp32 storage across ViT scales for the preconditioner
families considered in this scalability study.

For comparison, we also include a hypothetical \texttt{diag-kfac} structure,
which approximates each Kronecker factor by its diagonal:
\(\mathrm{Diag}(a)\,X\,\mathrm{Diag}(b)\) for a reshaped kernel
\(X\in\mathbb{R}^{m\times n}\).  This reduces the per-kernel preconditioner
storage from dense Kronecker-style scaling to \(O(m+n)\) scalars, providing a
lightweight option for trading geometric expressivity for memory.

\begin{table*}[t]
  \centering
  \footnotesize
  \setlength{\tabcolsep}{3pt}
  \caption{Cross-model analytical fp32 storage for e-KFAC and diag-KFAC preconditioners
        under ImageNet-scale ViT geometry.  Model counts include heads and
        classification layers; \#precond reports the number of stored preconditioner
        scalars for each parameterization.}
  \label{tab:model-scale-storage}
  \resizebox{\textwidth}{!}{
  \begin{tabular}{llcrrrr}
    \toprule
    Model & Storage & $S$ & \multicolumn{1}{c}{\#model} & \multicolumn{1}{c}{\#precond} & \multicolumn{1}{c}{precond} & \multicolumn{1}{c}{total} \\
     &  &  & \multicolumn{1}{c}{params} & \multicolumn{1}{c}{params} & \multicolumn{1}{c}{(MiB/dev.)} & \multicolumn{1}{c}{(MiB/dev.)} \\
    \midrule
    ViT-Ti/16 & e-kfac & 1 & 5{,}427{,}080 & 24{,}235{,}082 & 92.45 & 113.15 \\
    ViT-Ti/16 & e-kfac & 4 & 5{,}427{,}080 & 24{,}235{,}082 & 23.11 & 43.82 \\
    ViT-Ti/16 & diag-kfac & 1 & 5{,}427{,}080 & 73{,}426 & 0.28 & 20.98 \\
    ViT-Ti/16 & diag-kfac & 4 & 5{,}427{,}080 & 73{,}426 & 0.07 & 20.77 \\
    \midrule
    ViT-B/16 & e-kfac & 1 & 86{,}567{,}656 & 387{,}826{,}882 & 1{,}479.44 & 1{,}809.67 \\
    ViT-B/16 & e-kfac & 4 & 86{,}567{,}656 & 387{,}826{,}882 & 369.86 & 700.09 \\
    ViT-B/16 & diag-kfac & 1 & 86{,}567{,}656 & 294{,}038 & 1.12 & 331.35 \\
    ViT-B/16 & diag-kfac & 4 & 86{,}567{,}656 & 294{,}038 & 0.28 & 330.51 \\
    \midrule
    ViT-L/16 & e-kfac & 1 & 304{,}326{,}632 & 1{,}367{,}114{,}178 & 5{,}215.13 & 6{,}376.04 \\
    ViT-L/16 & e-kfac & 4 & 304{,}326{,}632 & 1{,}367{,}114{,}178 & 1{,}303.78 & 2{,}464.70 \\
    ViT-L/16 & diag-kfac & 1 & 304{,}326{,}632 & 771{,}990 & 2.94 & 1{,}163.86 \\
    ViT-L/16 & diag-kfac & 4 & 304{,}326{,}632 & 771{,}990 & 0.74 & 1{,}161.65 \\
    \bottomrule
  \end{tabular}
   }
\end{table*}

We also provide cadence results in Table~\ref{tab:vit-bl16-cadence-details} and
Fig.~\ref{fig:appendix-vit-cadence}, to be read as wall-clock overhead
measurements, not as accuracy or convergence comparisons.  As expected, the
cost decreases when \LLQR updates are performed less frequently, since fewer
preconditioner refreshes are executed per epoch. This steady-update scaling trend is further exposed in
Fig.~\ref{fig:lqr-time-complexity}.

\begin{table}[t]
  \centering
  \small
  \caption{Measured \LLQR cadence details for ViT-B/16 and ViT-L/16. Times are
compile-adjusted seconds per epoch, excluding the first preconditioner-update
overhead. Baseline rows report first-order optimizer runtimes; \LLQR rows
additionally report ratios relative to \textsc{SGD} and \SAM on the same
architecture.}
  \label{tab:vit-bl16-cadence-details}
  \resizebox{\textwidth}{!}{
  \begin{tabular}{lllrrrrr}
    \toprule
    Arch & Method & Update interval & Updates & Epoch time & vs. SGD & vs. SAM & Steady update \\
    \midrule
    ViT-B/16 & SGD & -- & -- & 33.20 & 1.00$\times$ & 0.35$\times$ & -- \\
    ViT-B/16 & AdamW & -- & -- & 38.51 & 1.16$\times$ & 0.41$\times$ & -- \\
    ViT-B/16 & SAM & -- & -- & 93.84 & 2.83$\times$ & 1.00$\times$ & -- \\
    \midrule
    ViT-B/16 & LLQR & 50 & 32 & 1007.31 & 30.34$\times$ & 10.73$\times$ & 59.38 \\
    ViT-B/16 & LLQR & 100 & 16 & 408.09 & 12.29$\times$ & 4.35$\times$ & 44.25 \\
    ViT-B/16 & LLQR & 200 & 8 & 204.11 & 6.15$\times$ & 2.18$\times$ & 39.38 \\
    ViT-B/16 & LLQR & 500 & 4 & 118.97 & 3.58$\times$ & 1.27$\times$ & 38.85 \\
    ViT-B/16 & LLQR & 1000 & 2 & 143.95 & 4.34$\times$ & 1.53$\times$ & 86.55 \\
    ViT-B/16 & LLQR & 1500 & 2 & 91.24 & 2.75$\times$ & 0.97$\times$ & 42.52 \\
    ViT-B/16 & LLQR+SAM & 50 & 32 & 813.98 & 24.51$\times$ & 8.67$\times$ & 46.05 \\
    ViT-B/16 & LLQR+SAM & 100 & 16 & 390.59 & 11.76$\times$ & 4.16$\times$ & 39.91 \\
    ViT-B/16 & LLQR+SAM & 200 & 8 & 426.81 & 12.85$\times$ & 4.55$\times$ & 85.33 \\
    ViT-B/16 & LLQR+SAM & 500 & 4 & 155.50 & 4.68$\times$ & 1.66$\times$ & 40.58 \\
    ViT-B/16 & LLQR+SAM & 1000 & 2 & 173.25 & 5.22$\times$ & 1.85$\times$ & 87.41 \\
    ViT-B/16 & LLQR+SAM & 1500 & 2 & 112.65 & 3.39$\times$ & 1.20$\times$ & 41.63 \\
    \midrule
    ViT-L/16 & SGD & -- & -- & 62.59 & 1.00$\times$ & 0.43$\times$ & -- \\
    ViT-L/16 & AdamW & -- & -- & 67.84 & 1.08$\times$ & 0.46$\times$ & -- \\
    ViT-L/16 & SAM & -- & -- & 146.54 & 2.34$\times$ & 1.00$\times$ & -- \\
    \midrule
    ViT-L/16 & LLQR & 50 & 32 & 1647.13 & 26.32$\times$ & 11.24$\times$ & 97.46 \\
    ViT-L/16 & LLQR & 100 & 16 & 868.47 & 13.88$\times$ & 5.93$\times$ & 97.34 \\
    ViT-L/16 & LLQR & 200 & 8 & 476.77 & 7.62$\times$ & 3.25$\times$ & 97.52 \\
    ViT-L/16 & LLQR & 500 & 4 & 275.63 & 4.40$\times$ & 1.88$\times$ & 95.18 \\
    ViT-L/16 & LLQR & 1000 & 2 & 299.48 & 4.79$\times$ & 2.04$\times$ & 96.46 \\
    ViT-L/16 & LLQR & 1500 & 2 & 301.14 & 4.81$\times$ & 2.06$\times$ & 95.91 \\
    ViT-L/16 & LLQR+SAM & 50 & 32 & 1809.84 & 28.92$\times$ & 12.35$\times$ & 95.78 \\
    ViT-L/16 & LLQR+SAM & 100 & 16 & 914.70 & 14.62$\times$ & 6.24$\times$ & 96.36 \\
    ViT-L/16 & LLQR+SAM & 200 & 8 & 667.77 & 10.67$\times$ & 4.56$\times$ & 96.20 \\
    ViT-L/16 & LLQR+SAM & 500 & 4 & 470.89 & 7.52$\times$ & 3.21$\times$ & 96.57 \\
    ViT-L/16 & LLQR+SAM & 1000 & 2 & 238.29 & 3.81$\times$ & 1.63$\times$ & 96.11 \\
    ViT-L/16 & LLQR+SAM & 1500 & 2 & 372.23 & 5.95$\times$ & 2.54$\times$ & 95.50 \\
    \bottomrule
  \end{tabular}
  }
\end{table}

\clearpage

\end{document}